\documentclass[twoside,11pt]{article}

\usepackage{times}
\usepackage{mathrsfs}
\usepackage{wrapfig}
\usepackage{booktabs}
\usepackage[bbgreekl]{mathbbol}

\usepackage{blindtext}

\usepackage[abbrvbib, preprint]{jmlr2e}

\usepackage{lastpage}
\jmlrheading{23}{2024}{1-\pageref{LastPage}}{08/24; Revised ---}{---}{21-0000}{Nathan Doumèche, Francis Bach, Gérard Biau, and Claire Boyer}

\ShortHeadings{Physics-informed kernel learning}{Doumèche, Bach, Biau, and Boyer}
\firstpageno{1}

\usepackage{amsmath}
\numberwithin{theorem}{section} 
\newtheorem{prop}[theorem]{Proposition}

\newtheorem{lem}[theorem]{Lemma}
\newtheorem{defi}[theorem]{Definition}

\usepackage{ulem}
\usepackage{color}
\usepackage[svgnames,dvipsnames]{xcolor}

\begin{document}

\title{Physics-informed kernel learning}

\author{%
 \name Nathan Doumèche \email nathan.doumeche@sorbonne-universite.fr \\
 \addr Sorbonne University, EDF R\&D
 \AND
 \name Francis Bach \email francis.bach@inria.fr \\
 \addr INRIA, \'Ecole Normale Supérieure,  PSL University
 \AND
 \name Gérard Biau \email gerard.biau@sorbonne-universite.fr\\
 \addr Sorbonne University, Institut universitaire de France
 \AND
 \name Claire Boyer \email claire.boyer@universite-paris-saclay.fr \\
 \addr Université Paris-Saclay, Institut universitaire de France
}

\editor{---}
\maketitle

\begin{abstract}
Physics-informed machine learning typically integrates physical priors into the learning process by minimizing a loss function 
that includes both a data-driven term and a partial differential equation (PDE) regularization. Building on the formulation of the problem as a kernel regression task, we use Fourier methods to approximate the associated kernel, and propose a tractable estimator that minimizes the physics-informed risk function. We refer to this approach as physics-informed kernel learning (PIKL). This framework provides theoretical guarantees, enabling the quantification of the physical prior’s impact on convergence speed. We demonstrate the numerical performance of the PIKL estimator through simulations, both in the context of hybrid modeling and in solving PDEs. In particular, we show that PIKL can outperform physics-informed neural networks in terms of both accuracy and computation time. Additionally, we identify cases where PIKL surpasses traditional PDE solvers, particularly in scenarios with noisy boundary conditions.
\end{abstract}


\section{Introduction}

\paragraph{Physics-informed machine learning.} Physics-informed machine learning (PIML), as described by \citet{raissi2019PINN}, is a promising framework that combines statistical and physical principles to leverage the strengths of both fields. PIML can be applied to a variety of problems, such as solving partial differential equations (PDEs) using machine learning techniques, leveraging PDEs to accelerate the learning of unknown functions (hybrid modeling), and learning PDEs directly from data (inverse problems). For an introduction to the field and a literature review, we refer to \citet{karniadakis2021piml} and \citet{cuomo2022scientific}. 

\paragraph{Hybrid modeling setting.}
We consider in this paper the classical regression model, which aims at learning the unknown function $f^\star : \mathbb{R}^d \to \mathbb{R}$ such that $Y = f^\star(X) + \varepsilon$, where $Y \in \mathbb{R}$ is the output, $X \in \Omega$ are the features with $\Omega \subseteq [-L, L]^d$ the input domain, and $\varepsilon$ is a random noise. 
Using $n$ observations $(X_1, Y_1), \ldots, (X_n, Y_n)$, independent copies of $(X, Y)$, the goal is to construct an estimator $ \hat{f}_n$ of $f^\star$. What makes PIML special compared to other regression settings is the prior knowledge that $f^\star$ approximately follows a PDE. Therefore, we assume that $f^\star$ is weakly differentiable up to the order $s > \frac{d}{2}$ and that there exists a known differential operator~$\mathscr{D}$ such that $\mathscr{D}(f^\star) \simeq 0$. This framework typically accounts for modeling error by recognizing that $\mathscr{D}(f^\star)$ may not be exactly zero, since most PDEs in physics are derived under ideal conditions and may not hold exactly in practice. 
For example, if $f^\star$ is expected to satisfy the wave equation $\partial^2_{t} f(x,t) \simeq \partial^2_{x} f(x,t)$, we define the operator $\mathscr{D}(f)(x,t) = \partial^2_{t} f(x,t) - \partial^2_{x} f(x,t)$ for $(x,t) \in \Omega$.

To estimate $f^\star$, we consider the minimizer of the physics-informed empirical risk 
\begin{equation}
    R_n(f) = \frac{1}{n}\sum_{i=1}^n |f(X_i) - Y_i|^2 + \lambda_n \|f\|^2_{H^s(\Omega)} + \mu_n \|\mathscr{D}(f)\|^2_{L^2(\Omega)}
    \label{eq:riskBase}
\end{equation}
over the class $\mathscr{F} = H^s(\Omega)$ of candidate functions, where $\lambda_n > 0 $ and $\mu_n \geqslant  0$ are hyperparameters that weight the relative importance of each term. Here, $H^s(\Omega)$ denotes the Sobolev space of functions with weak derivatives up to order $s$.
The empirical risk function $R_n(f)$ is characteristic of hybrid modeling, as it is composed of:
\begin{itemize}
\item A data fidelity term $\frac{1}{n}\sum_{i=1}^n |f(X_i) - Y_i|^2$, which is standard in supervised learning and measures the discrepancy between the predicted values $f(X_i)$ and the observed targets $Y_i$;
\item A	regularization term $\lambda_n \|f\|_{H^s(\Omega)}^2$, which penalizes the regularity of the estimator; 
\item A model error term $\mu_n \|\mathscr{D}(f)\|_{L^2(\Omega)}^2$, which measures the deviation of $f$ from the physical prior encoded in the differential operator $\mathscr{D}$.
To put it simply, the lower this term, the more closely the estimator aligns with the underlying physical principles.
\end{itemize}
Throughout the paper, we refer to $\hat{f}_n$ as the unique minimizer of the empirical risk function, i.e.,
\begin{equation}
    \hat f_n = \mathop{\mathrm{argmin}}_{f \in H^s(\Omega)}\; R_n(f).
    \label{eq:estimator_sob_0}
\end{equation}

\paragraph{Algorithms to solve the PIML problem.} 
Various algorithms have been proposed to compute the estimator $\hat f_n$, and physics-informed neural networks (PINNs) have emerged as a leading approach \citep[e.g.,][]{raissi2019PINN, arzani2021uncovering, karniadakis2021piml, kurz2022hybrid, Agharafeie2023from}. 
PINNs are usually trained by minimizing a discretized version of the risk over a class of neural networks using gradient descent strategies.
Leveraging the good approximation properties of neural networks, as the size of the PINN grows, this type of estimator typically converges to the unique minimizer over the entire space $H^s(\Omega)$ \citep{shin2020convergence,  doumeche2023convergence, mishra2022generalization, shin2023error, bonito2024convergenceerrorcontrolconsistent}. 
However, apart from the fact that optimizing PINNs by gradient descent is an art in itself, the theoretical understanding of the estimators derived through this approach is far from complete \citep{bonfanti2024challengesnonlinearregimephysicsinformed,  rathore2024challengestrainingpinnsloss}, and only a few initial studies have begun to outline their theoretical contours \citep[][]{krishnapriyan2021characterizing, wang2022when,doumeche2023convergence}.
Alternative algorithms for physics-informed learning have since been developed, primarily based on kernel methods, and are seen as promising candidates for bridging the gap between machine learning and PDEs. 
The connections between PDEs and kernel methods are now well established \citep[e.g.,][]{schaback2006from, chen2021solving, batlle2023error}. 
Recently, a kernel method has been adapted to perform operator learning \citep{nelsen2024operator}. It consists of solving a PDE using samples of the initial condition (with a purely data driven empirical risk).
In the case of hybrid modeling, including noise ($\varepsilon \neq 0$) and modeling error ($\mathscr D(f^\star) \neq 0$), \citet{doumeche2024physicsinformed} show that the PIML problem \eqref{eq:estimator_sob_0} can be reformulated as a kernel regression task. Provided the associated kernel $K$ is made explicit, this reformulation allows to obtain a closed-form estimator that converges at least at the Sobolev minimax rate. However, the kernel $K$ is highly dependent on the underlying PDE, and its computation can be tedious even for the the most simple priors, such as $\mathscr{D} = \frac{d}{dx}$ in one dimension.

\paragraph{Quantifying the impact of physics.} 
Understanding how physics can enhance learning is of critical importance to the PIML community. 
\citet{arnone2022spatialRegression} show that for second-order elliptic PDEs in dimension $d=2$, 
the PIML estimator converges at a rate of $n^{-4/5}$, outperforming the Sobolev minimax rate of $n^{-2/3}$. For general linear PDEs in dimension $d$, \citet{doumeche2024physicsinformed} adapt to PIML the notion of effective dimension, a central idea in kernel methods that quantify their convergence rate. In particular, for $d=1$, $s = 1$, $\Omega = [-L, L]$, and $\mathscr D = \frac{d}{dx}$, these authors show that the $L^2$-error of the physics-informed kernel method is of the order of ${\log(n)^2}/{n}$ when $\mathscr D(f^\star) = 0$, and achieves the Sobolev minimax rate $n^{-2/3}$ otherwise. However, extending this type of results to more complex differential operators $\mathscr D$ remains a challenge.

\paragraph{Contributions.}
Building on the characterization of the PIML problem as a kernel regression task, we use Fourier methods to approximate the associated kernel $K$ and, in turn, propose a tractable estimator minimizing the physics-informed risk function.
The approach involves developing the kernel $K$ along the Fourier modes with frequencies bounded my $m$, and then taking $m$ as large as possible.
We refer to this approach as the physics-informed kernel learning (PIKL) method. 
Subsequently, for general linear operators $\mathscr{D}$, a numerical strategy is developed to estimate the effective dimension of the kernel problem, allowing for the quantification of the expected statistical convergence rate when incorporating the physics prior into the learning process.   
Finally, we demonstrate the numerical performance of the PIKL estimator through simulations, both in the context of hybrid modeling and in solving partial differential equations. In short, the PIKL algorithm consistently outperforms specialized PINNs from the literature, which were specifically designed for the applications under consideration.

\section{The PIKL estimator}
\label{sec:finite_approx}
In this section, we detail the construction of the PIKL estimator, our approximate kernel method for physics-informed learning. We begin by observing that solving the PIML problem \eqref{eq:estimator_sob_0} is equivalent to performing a kernel regression task, as shown by \citet[][Theorem 3.3]{doumeche2024physicsinformed}. Thus, leveraging the extensive literature on kernel methods, it follows that the estimator $\hat{f}_n$ has the closed-form expression
\[
\hat f_n = \Big(x\mapsto (K(x, X_1), \hdots, K(x, X_n)) (\mathbb K + n I_n)^{-1} \mathbb Y\Big),
\]
where $K: \Omega^2 \to \mathbb{R}$ is the PIML kernel associated with the problem, and $\mathbb{K} \in \mathcal{M}_n(\mathbb R)$ is the kernel matrix defined by $\mathbb K_{i,j} = K(X_i, X_j)$. 

\paragraph{A finite-element-method approach.} 
The analysis of \cite{doumeche2024physicsinformed} reveals that the kernel related to the PIML problem is uniquely characterized as the solution to a weak PDE. Indeed, for all $x \in \Omega$, the function $y \mapsto K(x, y)$ is the unique solution in $H^s(\Omega)$ to the weak formulation
\begin{equation}
    \forall \phi \in H^s(\Omega), \quad \lambda_n  \int_{\Omega}\big[K(x, \cdot)\; \phi +\sum_{|\alpha| = d}\partial^\alpha K(x, \cdot)\; \partial^\alpha  \phi\big] + \mu_n \int_{\Omega}\mathscr{D}(K(x, \cdot)) \; \mathscr{D}(\phi) = \phi(x).
    \label{eq:weakPDE}
\end{equation}
A spontaneous idea is to approximate the kernel $K$ using finite element methods (FEM). For illustrative purposes, we have applied this approach in numerical experiments with $d=1$, $\Omega = [0,1]$, and $\mathscr{D}(f) = \frac{d}{dx}f - f$. Figure \ref{fig:fem} (Left) depicts the associated kernel function $K(0.4, \cdot)$ with $\lambda_n = 10^{-2}$, $\mu_n = 0$, and $100$ nodes. Figure \ref{fig:fem} (Right) shows that the PIML method \eqref{eq:estimator_sob_0} successfully reconstructs $f^\star(x) = \exp(x)$ using $n = 10$ data points, $\varepsilon \sim \mathcal{N}(0, 10^{-2})$, $\lambda_n = 10^{-10}$, and $\mu_n = 1000$. However, solving the weak formulation \eqref{eq:weakPDE} in full generality is quite challenging, particularly when dealing with arbitrary domains $\Omega$ in dimension $d>1$. In fact, FEM strategies need to be specifically tailored to the PDE and the domain in question. Additionally, standard kernel methods combined with FEM approaches come at a high computational cost, since storing the matrix $\mathbb{K}$ requires $O(n^2)$ memory. This becomes prohibitive for large amounts of data, as $n = 10^4$ already requires several gigabytes of RAM. 
\begin{figure}
    \centering
    \includegraphics[width = 0.4\textwidth]{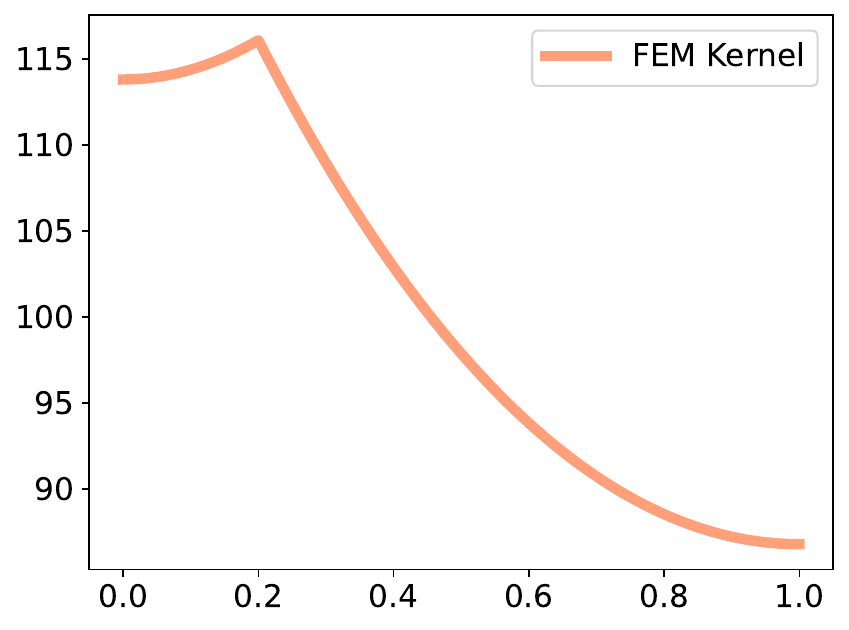}
    \includegraphics[width = 0.4\textwidth]{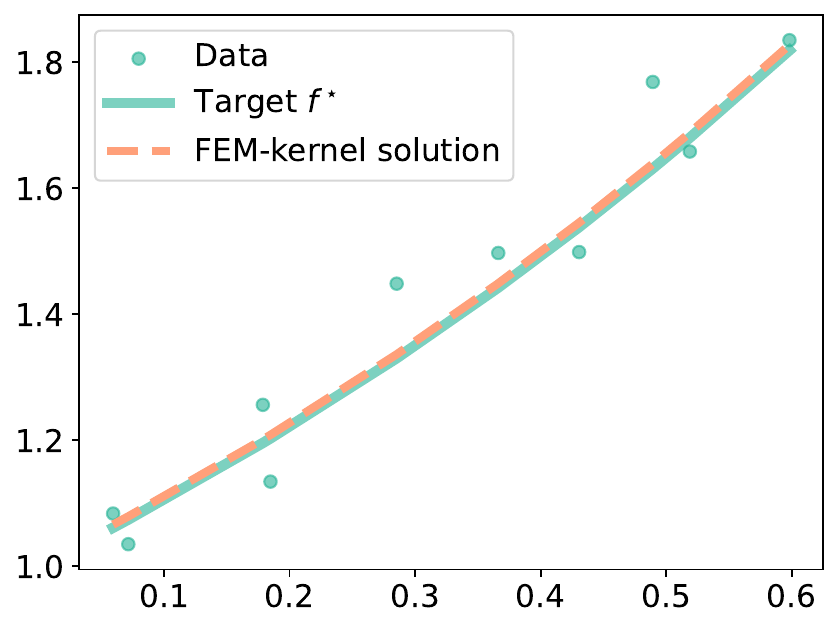}
    \caption{\textbf{Left:} Kernel function $K(0.4, \cdot)$ estimated by the FEM. \textbf{Right:} Kernel method $\hat f_n$ combined with the FEM.}
    \label{fig:fem}
\end{figure}

\paragraph{Fourier approximation.}
Our primary objective in this article is to develop a more agile, flexible, and efficient method capable of handling arbitrary domains $\Omega$. To achieve this, a natural approach is to expand the kernel $K$ as a truncated Fourier series, i.e.,
$K_m(x,y) = \sum_{\|k\|_\infty \leqslant m} a_k \phi_k(x)\phi_k(y)$, 
with $(\phi_k)_{\|k\|_\infty \leqslant m}$ the Fourier basis, $(a_k)_{\|k\|_\infty \leqslant m}$ the kernel coefficients in this basis, and $m$ the order of approximation.
This idea is at the core of techniques such as random Fourier features (RFF) \citep[e.g.,][]{rahimi2007random, yang2012nystrom}. 
However, unlike RFF, the Fourier features in our problem are not random quantities, as they systematically correspond to the low-frequency modes. 
This low-frequency approximation is particularly well-suited to the Sobolev penalty, which more strongly regularizes high frequencies (the analogous RFF algorithm would involve sampling random frequencies $k$ according to a density that is proportional to the Sobolev decay). 
In addition, and more importantly, the use of such approximations bypasses the need to discretize the domain into finite elements and requires only the knowledge of the (partial) Fourier transform of $\mathbf{1}_\Omega$, as will be explained later.

So, following \citet{doumeche2024physicsinformed}, our PIKL algorithm first requires extending the learning problem from $\Omega\subseteq [-L,L]^d$ to the torus $[-2L,2L]^d$. This initial technical step allows us to use approximations with the standard Fourier basis, given for $k\in\mathbb{Z}^d$ and $x\in [-2L,2L]^d$ by
\[
\phi_k(x)=(4L)^{-d/2} e^{\frac{i \pi}{2L}\langle k, x\rangle},
\]
particularly adapted to periodic functions on $[-2L,2L]^d$. 
The minimization of the risk $R_n$ from \eqref{eq:riskBase} over $H^s(\Omega)$ can be then transferred into the minimization of the PIML risk 
\begin{equation}
    \bar R_n(f) = \frac{1}{n}\sum_{i=1}^n |f(X_i) - Y_i|^2 + \lambda_n \|f\|^2_{H^s_{\mathrm{per}}([-2L,2L]^d)} + \mu_n \|\mathscr{D}(f)\|^2_{L^2(\Omega)}
    \label{eq:risk_function}
\end{equation}
over the periodic Sobolev space $H^s_{\mathrm{per}}([-2L,2L]^d)$.
The kernel underlying \eqref{eq:risk_function} is determined by the RKHS norm
\[\|f\|^2_{\mathrm{RKHS}} = \lambda_n \|f\|_{H^s_{\mathrm{per}}([-2L,2L]^d)}^2 + \mu_n \|\mathscr D(f)\|_{L^2(\Omega)}^2.\] 
It is important to note that the estimators derived from the minimization of either $R_n$ or $\bar{R}_n$ share the same statistical guarantees, as both kernel methods have been shown to converge to $f^\star$ at the same rate \citep[][Theorem 4.6]{doumeche2024physicsinformed}. 

Since the kernel $K$ associated with problem \eqref{eq:estimator_sob_0} is most often intractable, a key milestone in the development of our method is to 
minimize $\bar{R}_n$ not over the entire space $H^s_\mathrm{per}([-2L,2L]^d)$, but rather on the finite-dimensional Fourier subspace $H_m=\mathrm{Span}((\phi_k)_{\|k\|_\infty \leqslant m})$. This leads to the PIKL estimator, defined by
\begin{align}
\hat{f}^{\mathrm{PIKL}} = \mathop{\mathrm{argmin}}_{f \in H_m}\; \bar R_n(f).
\label{pb:PIML_on_Hm}
\end{align}
This naturally transforms the PIML problem into a finite-dimensional kernel regression task, where the associated kernel $K_m$ corresponds to a Fourier expansion of $K$, as will be clarified in the following paragraph. Of course, $H_m$ provides better approximates of $H^s_\mathrm{per}([-2L,2L]^d)$
as $m$ increases, since for any function $f \in H^s_\mathrm{per}([-2L,2L]^d)$, $\lim_{m\to\infty}\min_{g\in H_m}\|f-g\|_{H^s_\mathrm{per}([-2L,2L]^d)} = 0$. 
Remarkably, the key advantage of using Fourier approximations in our PIKL algorithm lies in the fact that
both the Sobolev norm $\|f\|_{H^s_\mathrm{per}([-2L,2L]^d)}$ and the PDE penalty $\|\mathscr{D}(f)\|_{L^2(\Omega)}$ are bilinear functions of the Fourier coefficients of $f$. As shown below, these bilinear forms can be represented as closed-form matrices, easing the computation of the estimator.

\paragraph{RKHS norm in Fourier space.} Suppose that the differential operator $\mathscr{D}$ is linear with constant coefficients, i.e., it can be expressed as
$\mathscr{D}(f) = \sum_{|\alpha| \leqslant s} a_\alpha \partial^\alpha f$ for some $s \in \mathbb{N}^\star$ and $a_\alpha \in \mathbb{R}$.
If $f \in H_m$, then $f$ can be rewritten in terms of its Fourier coefficients as
\[
f(x) = \langle z, \Phi_m(x) \rangle_{\mathbb C^{(2m+1)^d}},
\]
where $\langle \cdot, \cdot \rangle_{\mathbb C^{(2m+1)^d}}$ denotes the canonical inner product on $\mathbb C^{(2m+1)^d}$, $z$ is the vector of Fourier coefficients of $f$, and 
\[
\Phi_m(x) = \Big(x \mapsto (4L)^{-d/2} e^{\frac{i \pi}{2L}\langle k, x\rangle}\ \Big)_{\|k\|_\infty\leqslant m}.
\] 
According to Parseval's theorem, the $L^2$-norm of the derivatives of $f\in H^s_{\mathrm{per}}([-2L,2L]^d)$ can be expressed using the Fourier coefficients of $f$ as follows: for $r\leqslant s$ and $1 \leqslant i_1, \hdots, i_r \leqslant d$, 
\[ 
\|\partial^r_{i_1, \hdots, i_r}f\|_{L^2([-2L,2L]^d)}^2 = (2L)^{-2k}\sum_{\|j\|_\infty\leqslant m} |z_j|^2 \prod_{\ell=1}^r j_{i_\ell}^2.
\]
With this notation, the Sobolev norm reads
\[
\|f\|_{H^s_{\mathrm{per}}([-2L,2L]^d)}^2 
= \sum_{\|j\|_\infty\leqslant m,\|k\|_\infty\leqslant m} z_j \bar{z_k} \Big(1+\Big(\frac{\|k\|_2^2}{(2L)^d}\Big)^{s}\Big)\delta_{j,k},
\]
and, similarly,
\[\|\mathscr D(f)\|_{L^2(\Omega)}^2 = \sum_{\|j\|_\infty\leqslant m,\|k\|_\infty\leqslant m} z_j \bar{z_k} \frac{P(j)\bar P(k)}{(4L)^d}\int_\Omega e^{\frac{i \pi}{2L}\langle k-j, x\rangle}dx,\]
where $P(k) = \sum_{|\alpha| \leqslant s} a_\alpha (\frac{-i\pi}{2L})^{|\alpha|}\prod_{\ell =1}^d  (k_\ell)^{\alpha_\ell} $. 
Therefore, introducing $M_m \in \mathcal{M}_{(2m+1)^d}(\mathbb C)$ the matrix with coefficients indexed by $j,k\in \{-m, \hdots, m\}^d$,
\begin{equation}
    (M_m)_{j,k} = \lambda_n \Big(1+\Big(\frac{\|k\|_2^2}{(2L)^d}\Big)^{s}\Big)\delta_{j,k}+\mu_n \frac{P(j)\bar P(k)}{(4L)^d}\int_\Omega e^{\frac{i \pi}{2L}\langle k-j, x\rangle}dx,
    \label{eq:Mm}
\end{equation}
we obtain that the RKHS norm of $f$ is expressed as a bilinear form of its Fourier coefficients $z$, i.e.,  
\[
\|f\|^2_{\mathrm{RKHS}} = \langle z, M_mz\rangle_{\mathbb C^{(2m+1)^d}}.
\]
It is important to note that $M_m$ is Hermitian,\footnote{since $M_m^\star = \bar M_m^\top = M_m$.} positive,\footnote{since $\langle z, M_mz\rangle_{\mathbb C^{(2m+1)^d}} = \|f\|^2_{\mathrm{RKHS}} \geqslant  0$.} and definite.\footnote{since $\langle z, M_mz\rangle_{\mathbb C^{(2m+1)^d}} = 0$ implies $ \|f\|_{H^s([-2L,2L]^d)} = 0$, i.e., $f = 0$.} Therefore, the spectral theorem (see Theorem~\ref{thm:spectral}) ensures that $M_m$ is invertible, and that its positive inverse square root $M_m^{-1/2}$ is unique and well-defined. We have now all the ingredients to define the PIKL algorithm.

\begin{remark}[Linear PDEs with non-constant coefficients]
    This framework could be adapted to PDEs with non-constant coefficients, i.e., to operators $\mathscr{D}(f) = \sum_{|\alpha| \leqslant s} a_\alpha \partial^\alpha f$ for some $s \in \mathbb{N}^\star$ and $a_\alpha \in C^0(\mathbb{R})$. In this case, the polynomial $P$ in \eqref{eq:Mm} should be replaced by convolutions involving the Fourier coefficients of the functions $a_\alpha$.
    \end{remark}

\paragraph{Computing the PIKL estimator.} 
For a function $f\in H_m$, one can evaluate $f$ at $x$ by $f(x) = \langle M_m^{1/2}z, M_m^{-1/2}\Phi_m(x)\rangle_{\mathbb C^{(2m+1)^d}}$. This reproducing property indicates that minimizing the risk $\bar R_n$ on $H_m$ is a kernel method governed by the kernel
\[
K_m(x,y) = \langle  M_m^{-1/2} \Phi_m(x), M_m^{-1/2}\Phi_m(y)\rangle_{\mathbb C^{(2m+1)^d}}.
\]
Define $\mathbb Y = (Y_1,\hdots,Y_n)^\top$ and $\mathbb K_m \in \mathcal M_n(\mathbb C)$ to be the matrix such that $(\mathbb K_m)_{i,j} = K_m(X_i, X_j)$ for all $1 \leqslant i, j \leqslant n$. 
The PIKL estimator \eqref{pb:PIML_on_Hm}, minimizer of $\bar{R}_n$ restricted to $H_m$, is therefore given by
\begin{align}           \hat{f}^{\mathrm{PIKL}} (x)
&= (K_m(x, X_1), \hdots, K_m(x, X_n)) (\mathbb K_m + n I_n)^{-1} \mathbb Y \nonumber\\
&= \Phi_m(x)^\star (\mathbb{\Phi}^\star\mathbb{\Phi} + n M_m)^{-1} \mathbb \Phi^\star \mathbb Y,\label{eq:PIKL}
\end{align}
where 
$
\mathbb{\Phi} = \begin{pmatrix}
    \Phi_m(X_1)^\star\\
    \vdots\\
    \Phi_m(X_n)^\star
\end{pmatrix} \in  \mathcal{M}_{n, (2m+1)^d}(\mathbb C)$. The formula obtained in \eqref{eq:PIKL} is provided by the so-called kernel trick.  
This step offers a significant advantage to the PIKL estimator as it reduces the computational burden in large sample regimes: instead of storing and inverting the $n \times n$ matrix $\mathbb{K}_m + nI_n$, we only need to store and invert the $(2m+1)^d \times (2m+1)^d$ matrix $\mathbb{\Phi}^\star \mathbb{\Phi} + n M_m$. 
Moreover, the computation of $\mathbb{\Phi}^\star \mathbb{\Phi}$ and $\mathbb \Phi^\star \mathbb Y$ can be performed online and in parallel as $n$ grows.
Of course, this approach is subject to the curse of dimensionality. However, it is unreasonable to try to learn more parameters than the sample complexity $n$. Therefore, in practice, $(2m + 1)^d \ll n$, which justifies the preference of the $(2m + 1)^{2d} $ storage complexity over the $n^2$ storage complexity of the FEM-based algorithm. In addition, similar to PINNs, the PIKL estimator has the advantage that its training phase takes longer than its evaluation at certain points. In fact, once the $(2m + 1)^d$ Fourier modes of $\hat{f}_n(x)$ (given by $(\mathbb{\Phi}^\star\mathbb{\Phi} + nM_m)^{-1}\mathbb{\Phi}^\star Y$) are computed, the evaluation of $\Phi^\star_m(x)$ is straightforward. This is in sharp contrast to the FEM-based strategy, which requires approximating the kernel vector $(K(x, X_1), \dots, K(x, X_n))$ at each query point $x$.

We also emphasize that the PIKL predictor is characterized by low-frequency Fourier coefficients, which, in turn, enhance its interpretability.  This methodology differs significantly from PINNs, which are less interpretable and rely on gradient descent for optimization \citep[see, e.g.,][]{wang2022when}.

\begin{remark}[PIKL vs.\ spectral methods] In the PIML context, the RFF approach resembles a well-known class of powerful tools for solving PDEs, known as spectral and pseudo-spectral methods \citep[e.g.,][]{canuto2007spectral}. These methods solve PDEs by selecting a basis of orthogonal functions and computing the coefficients of the solution on that basis to satisfy both the boundary conditions and the PDE itself. For example, the Fourier basis $(x \mapsto \exp(\frac{i \pi}{2L} \langle k, x\rangle (2L)^{-1}))_{k \in \mathbb{Z}^d}$ already used in this paper is particularly well suited for solving linear PDEs on the square domain $[-2L, 2L]^d$ with periodic boundary conditions. Spectral methods such as these have already been used in the PIML community to integrate PDEs with machine learning techniques \citep[e.g.,][]{meuris2023machine}. However, the basis functions used in spectral and pseudo-spectral methods must be specifically tailored to the domain $\Omega$, the differential operator $\mathscr{D}$, and the boundary conditions.
For more information on this topic, please refer to Appendix~\ref{sec:discussion_spec}.
\end{remark}

\paragraph{Computing $M_m$ for specific domains.} Computing the matrix $M_m$ requires the evaluation of the integrals $(j,k) \mapsto \int_\Omega e^{\frac{i \pi}{2L}\langle k-j, x\rangle}dx$.
In general, these integrals can be approximated using numerical integration schemes or Monte Carlo methods. However, it is possible to provide closed-form expressions for specific domains $\Omega$. 
To do so, for $d \in \mathbb{N}^\star$, $L > 0$, and $\Omega \subseteq [-L, L]^d$, we define the characteristic function $F_\Omega $ of $\Omega$ by 
    \[
    F_\Omega(k) =  \frac{1}{(4L)^{d}}\int_\Omega e^{\frac{i \pi}{2L}\langle k, x\rangle}dx.
    \] 
\begin{prop}[Closed-form characteristic functions]
The characteristic functions associated with the cube and the Euclidean ball can be analytically obtained as follows.
\begin{itemize}
\item (Cube) Let $\Omega = [-L,L]^d$. Then, for $k\in\mathbb{Z}^d$,
    \[F_\Omega(k) =  \prod_{j=1}^d\frac{\sin(\pi k_j/2)}{\pi k_j}.\]
    \label{prop:geom_cube}
\item (Euclidean ball) Let $d=2$ and $\Omega = \{x \in [-L,L], \|x\|_2 \leqslant L\}$. Then, for $k\in\mathbb{Z}^d$,
    \[F_\Omega(k) =   \frac{J_1(\pi \|k\|_2/2)}{4\|k\|_2},\]
where $J_1$ is the  Bessel function of the first kind of parameter 1.   
\end{itemize}
\end{prop}
This proposition, along with similar analytical results for other domains, can be found in \citet[][Table 13.4]{bracewell}, noting that
$F_\Omega$
is the Fourier transform of the indicator function $\mathbf{1}_\Omega$ and is also the characteristic function of the uniform distribution on $\Omega$ evaluated at $\frac{k}{2L}$.
We can extend these computations further since, given the characteristic functions of elementary domains $\Omega$, it is easy to compute the characteristic functions of translation, dilation, disjoint unions, and Cartesian products of such domains
(see Proposition~\ref{prop:op_char} in Appendix~\ref{app:Theory}). 
For instance, it is straightforward to obtain the characteristic function of the three-dimensional  cylinder $\Omega = \{x \in [-L,L], \|x\|_2 \leqslant L\} \times [-L,L]$ as
\[
F_\Omega(k_1, k_2, k_3) = \frac{J_1(\pi (k_1^2+k_2^2)^{1/2}/2)}{4(k_1^2+k_2^2)^{1/2}} \times \frac{\sin(\pi k_3/2)}{\pi k_3}.
\]

\section{The PIKL algorithm in practice}

\begin{wrapfigure}[15]{r}{0.4\textwidth}
    \centering
    \includegraphics[width=\linewidth]{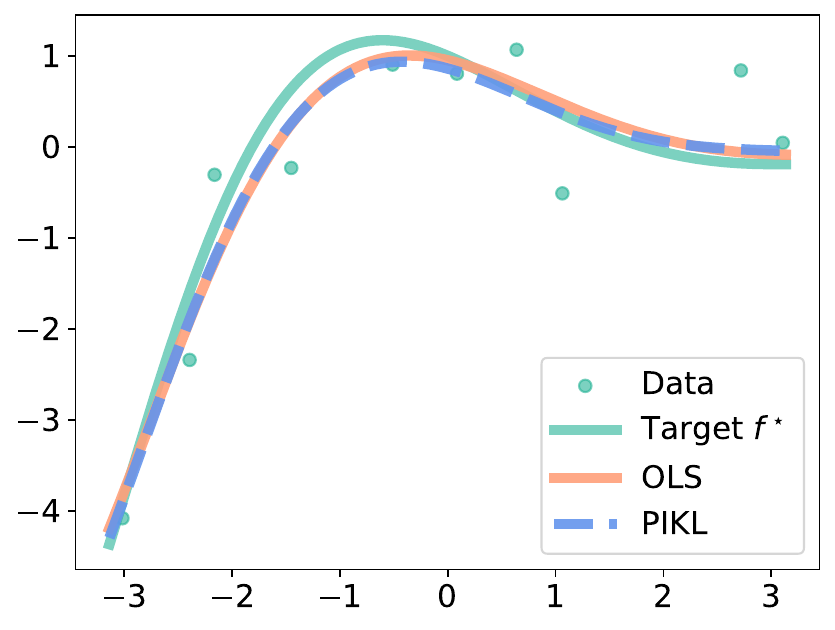}
    \caption{OLS and PIKL estimators for the harmonic oscillator with $d=1$, sample size $n=10$.
    \label{fig:ressort_least_square}}
\end{wrapfigure}
To enhance the reproducibility of our work, we provide a \texttt{Python} package that implements the PIKL estimator, designed to handle any linear PDE prior with constant coefficients in dimensions $d = 1$ and $d = 2$. This package is available at \url{https://github.com/NathanDoumeche/numerical_PIML_kernel}. $\;\;$ Note that this package implements the matrix inversion of the PIKL formula \eqref{eq:PIKL} by solving a linear system using the LU decomposition. Of course, any other efficient method to avoid direct matrix inversion could be used  instead, such as solving a linear system with the conjugate gradient method. 

Through numerical experiments, we demonstrate the performance of our approach in simulations for hybrid modeling (Subsection~\ref{sec:hybrid_mod}),
and derive experimental convergence rates that quantify the benefits of incorporating PDE knowledge into a learning regression task (Subsection~\ref{sec:eff_dim}).

\subsection{Hybrid modeling}
\label{sec:hybrid_mod}

\paragraph{Perfect modeling with closed-form PDE solutions.} 
We start by assessing the performance of the PIKL estimator in a perfect modeling situation (i.e., $\mathscr D(f^\star) = 0$), where the solutions of the PDE $\mathscr D(f) = 0$ can be decomposed on a basis $(f_k)_{k\in \mathbb N}$ of closed-form solution functions. 
In this ideal case, the spectral method suggests an alternative estimator, which involves learning the coefficients $a_k \in \mathbb R$ of $f^\star = \sum_{k\in \mathbb N} a_k f_k$ in this basis.
For example, consider the one-dimensional case ($d = 1$) with domain $\Omega = [-\pi, \pi]$, and the harmonic oscillator differential prior $\mathscr D(f) = \frac{d^2 f}{dx^2} + \frac{df}{dx} + f$. In this case, the solutions of $\mathscr D(f) = 0$ are the linear combinations $f = a_1 f_1 + a_2 f_2$, where  $(a_1, a_2) \in \mathbb R^2$, $f_1(x) = \exp(-x/2)\cos(\sqrt{3}x/2)$, and $f_2(x) = \exp(-x/2)\sin(\sqrt{3}x/2)$. 
Thus, the spectral method focuses on learning the vector $(a_1, a_2)\in \mathbb R^2$, instead of learning the Fourier coefficients of $f^\star$, which is the approach taken by the PIKL algorithm.

A baseline that exactly leverages the particular structure of this problem, referred to as the ordinary least squares (OLS) estimator, is therefore
$\hat g_n = \hat a_1 f_1 + \hat a_2 f_2$, where 
\[
(\hat a_1, \hat a_2) = \mathop{\mathrm{argmin}}_{(a_1, a_2)\in \mathbb R^2}\; \frac{1}{n}\sum_{i=1}^n |a_1 f_1(X_i)+ a_2 f_2(X_i)-Y_i|^2.
\]
\begin{wrapfigure}[20]{r}{0.4\textwidth}
    \centering
    \vspace{-0.5cm}
\includegraphics[width=\linewidth]{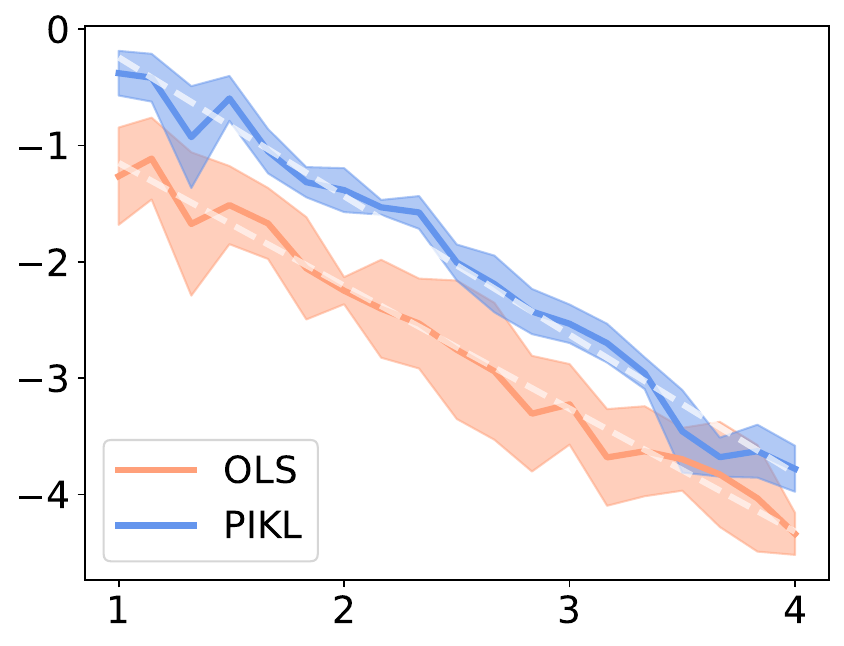}
     \caption{$L^2$-error (mean $\pm$ std over 5 runs) of the OLS and PIKL estimators for the harmonic oscillator with $d=1$, w.r.t.\ $n$ in $\log_{10}-\log_{10}$ scale.  The dashed lines represent adjusted linear models w.r.t.\ $n$, for both $L^2$-errors.\label{fig:perf_ressort_least_square}}
\end{wrapfigure}To compare the PIKL and OLS estimators, we generate data such that $Y = f^\star(X) + \varepsilon$, where $X \sim \mathcal{U}(\Omega)$, $\varepsilon \sim \mathcal{N}(0, \sigma^2)$ with $\sigma = 0.5$, and the target function is $f^\star = f_1$ (corresponding to $(a_1, a_2) = (1, 0)$). We implement the PIKL algorithm with $601$ Fourier modes ($m = 300$) and $s = 2$. 
Figure~\ref{fig:ressort_least_square} shows that even with very few data points ($n = 10$) and high noise levels, both the OLS and PIKL methods effectively reconstruct $f^\star$, both incorporating physical knowledge in their own way. In Figure~\ref{fig:perf_ressort_least_square}, we display the $L^2$-error of both estimators for different sample sizes $n$. The two methods have an experimental convergence rate of $n^{-1.1}$, which is consistent with the expected parametric rate of $n^{-1}$.
This sanity check shows that under perfect modeling conditions, the PIKL estimator with $m=300$ performs as well as the OLS estimator specifically designed to explore the space of PDE solutions.

\paragraph{Combining the best of physics and data in imperfect modeling.}
In this paragraph, we deal with an imperfect modeling scenario using the heat differential operator $\mathscr D(f) = \partial_1 f - \partial^2_{2,2} f$ in dimension $d=2$ over the domain $\Omega = [-\pi, \pi]^2$. The data are generated according to the model $Y = f^\star(X) + \varepsilon$, where $\|\mathscr D(f^\star)\|_{L^2(\Omega)} \neq 0$. 
We assume, however, that the PDE serves as a good physical prior, meaning that $\|f^\star\|_{L^2(\Omega)}^2$ is significantly larger than the modeling error $\|\mathscr D(f^\star)\|_{L^2(\Omega)}^2$.
The hybrid model is implemented using the PIKL estimator with parameters $s=2$, $\lambda_n = n^{-2/3}/10$, and $\mu_n = 100/n$. These hyperparameters are selected to ensure that, when only a small amount of data is available, the model relies heavily on the PDE. 
Yet, as more data become available, the model can use the data to correct the modeling error. 
The performance of the PIKL estimator is compared with that of a purely data-driven estimator, referred to as the Sobolev estimator, and a strongly PDE-penalized estimator, referred to as the PDE estimator. 
The Sobolev estimator uses the same parameter $s=2$ and $\lambda_n = n^{-2/3}/10$, but sets $\mu_n = 0$. This configuration ensures that the estimator relies entirely on the data without considering the PDE as a prior. On the other hand, the PDE estimator is configured with parameters $s=2$, $\lambda_n = 10^{-10}$, and $\mu_n = 10^{10}$. These hyperparameters are set to ensure that the resulting PDE estimator effectively satisfies the heat equation, making it highly dependent on the physical model.

We perform an experiment where $\varepsilon \sim \mathcal{N}(0, \sigma^2)$ with $\sigma = 0.5$, and $f^\star(t,x) = \exp(-t)\cos(x) + 0.5 \sin(2x)$. This scenario is an example of imperfect modeling, since $\|\mathscr{D}(f^\star)\|_{L^2(\Omega)}^2 = \pi > 0$. However, the heat equation serves as a strong physical prior, since 
$\|\mathscr{D}(f^\star)\|_{L^2(\Omega)}^2/ \|f^\star\|_{L^2(\Omega)}^{2}\simeq 4 \times 10^{-3}$.
Figure~\ref{fig:expe_heat} illustrates the performance of the different estimators. 

\begin{wrapfigure}[22]{r}{0.5\linewidth}
    \centering
    \includegraphics[width=\linewidth]{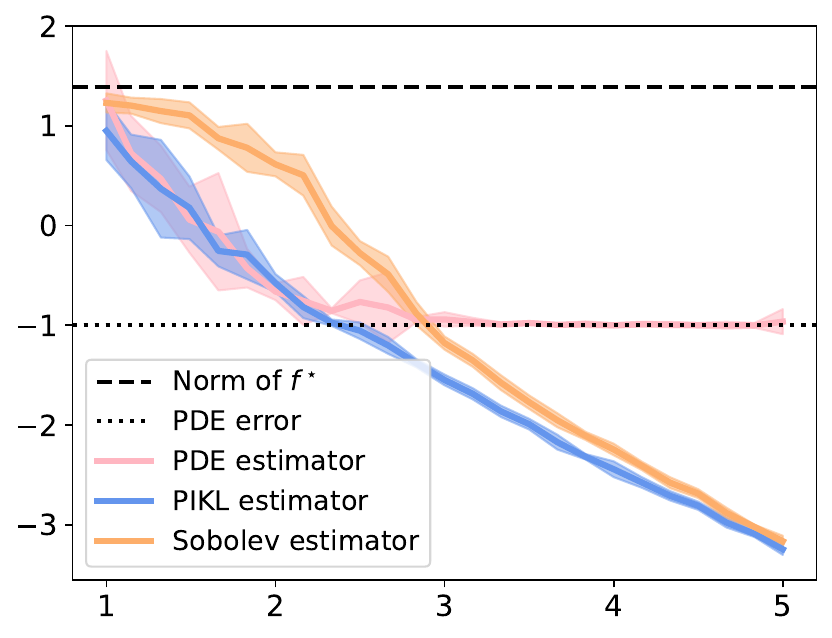}
    \caption{$L^2$-error  (mean $\pm$ std over 5 runs) of the PDE, PIKL, and Sobolev estimators  for imperfect modeling with the heat equation, as a function of $n$ in $\log_{10}-\log_{10}$ scale. The PDE error is the $L^2$-norm between $f^\star$ and the PDE solution that is closest to $f^\star$.}
    \label{fig:expe_heat}
\end{wrapfigure}
Clearly, the PDE estimator outperforms the Sobolev estimator when the data set is small ($n \leqslant 10^2$). 
As expected, the performance of the Sobolev estimator improves as the sample size increases ($n \geqslant  10^3$), but it remains consistently inferior to that of the PIKL.
When only a small amount of data is available, the PDE provides significant benefits, and the $L^2$-error decreases at the super-parametric rate of $n^{-2}$ for both the PIKL and the PDE estimators. However, in the context of imperfect modeling, the PDE estimator cannot overcome the PDE error, resulting in no further improvement beyond $n \geqslant  100$. In addition, when a large amount of data is available, the data become more reliable than the PDE. In this case, the errors for both the PIKL and the Sobolev estimators decrease at the Sobolev minimax rate of $n^{-2/3}$. Overall, the PIKL estimator successfully combines the strengths of both approaches, using the PDE when data is scarce and relying more on data when it becomes abundant.

\subsection{Measuring the impact of physics with the effective dimension}
\label{sec:eff_dim}
The important question of measuring the impact of the differential operator $\mathscr D$ on the convergence rate of the PIML estimator has not yet found a clear answer in the literature. In this subsection, we propose an approach to experimentally compare the PIKL convergence rate to the Sobolev minimax rate in $H^s(\Omega)$, which is $n^{-2s/(2s+d)}$ \citep[e.g.,][Theorem 2.1]{tsybakov2009introduction}. 
\paragraph{Theoretical backbone.}
According to \citet[Theorem 4.3]{doumeche2024physicsinformed}, if $X$ has a bounded density and the noise $\varepsilon$ is sub-Gamma with parameters $(\sigma, M)$, the $L^2$-error of both estimators \eqref{eq:estimator_sob_0} and \eqref{pb:PIML_on_Hm} satisfies
\begin{align}
     &\mathbb{E}\int_\Omega |\hat f_n-f^\star|^2 d{\mathbb P}_X \nonumber\\
        &\quad \leqslant  C_4 \log^2(n)\Big(\lambda_n \|f^\star\|_{H^s(\Omega)}^2 + \mu_n \|\mathscr{D}(f^\star)\|_{L^2(\Omega)}^2 + \frac{M^2}{n^2 \lambda_n} + \frac{\sigma^2\mathscr{N}(\lambda_n, \mu_n)}{n}\Big), \label{eq:err_l2}
\end{align} 
where $\mathbb P_X$ is the distribution of $X$. The quantity $\mathscr{N}(\lambda_n, \mu_n)$ on the right-hand side of~\eqref{eq:err_l2} is referred to as the effective dimension \citep[see, e.g.,][]{caponnetto2007optimal}. Since $\lambda_n$ and $\mu_n$ can be freely chosen by the practitioner, the effective dimension $\mathscr{N}(\lambda_n, \mu_n)$ becomes a key consideration that help quantify the impact of the physics on the learning problem. 
Unfortunately, bounding $\mathscr{N}(\lambda_n, \mu_n)$ is not trivial. 
\citet{doumeche2024physicsinformed} have shown that 
\[
\mathscr N(\lambda_n, \mu_n) \leqslant  \sum_{\lambda\in \sigma(C\mathscr O_nC)} \frac{1}{1+\lambda^{-1}},
\]
 where $\mathscr O_n$ is the operator $\mathscr O_n = \lim_{m\to \infty} M_m^{-1}$ (where the limit is taken in the sense of the operator norm --- see Definition~\ref{defi:op_norm}) and $C$ is the operator $C(f) = 1_\Omega f$. Therefore, a natural idea to assess the effective dimension is to replace $C\mathscr O_nC$ by $C_mM_m^{-1}C_m$, where $C_m: H_m \to H_m$ is defined by \[\forall j,k\in \{-m, \hdots, m\}^d, \quad (C_m)_{j,k} = \frac{1}{(4L)^{d}}\int_\Omega e^{\frac{i \pi}{2L}\langle k, x\rangle}dx.\]
The following theorem shows that this is a sound strategy, in the sense that computing the effective dimension using the eigenvalues of $C_mM_m^{-1}C_m$ becomes increasingly accurate as $m$ grows.
\begin{theorem}[Convergence of the effective dimension]
\label{thm:convergence_eff_dim}\hfill
    \begin{itemize}
        \item[(i)] One has
        \[\lim_{m\to\infty}\sum_{\lambda\in \sigma(C_mM_m^{-1}C_m)} \frac{1}{1+\lambda^{-1}} = \sum_{\lambda\in \sigma(C\mathscr O_nC)} \frac{1}{1+\lambda^{-1}}.\]
        \item[(ii)] Let $\sigma^\downarrow_k(C_mM_m^{-1}C_m)$ be the $k$-th highest eigenvalue of $C_mM_m^{-1}C_m$. The spectrum of the matrix $C_mM_m^{-1}C_m$ converges to the spectrum of $C\mathscr O_nC$ in the following sense:
        \[\forall k \in \mathbb N^\star,\quad \lim_{m\to\infty} \sigma^\downarrow_k(C_mM_m^{-1}C_m) = \sigma^\downarrow_k(C\mathscr O_nC).\]
    \end{itemize}
\end{theorem}

The provided \texttt{Python} package\footnote{\url{https://github.com/NathanDoumeche/numerical_PIML_kernel}} includes numerical approximations of the effective dimension in dimensions $d=1$ and $d=2$ for any linear operator $\mathscr{D}$ with constant coefficients, when $\Omega$ is either a cube or a Euclidean ball. The code is available is designed to run on both CPU and GPU. The convergence of the effective dimension as $m$ grows is studied in greater detail in Appendix \ref{sec:eff_dim_m}.

\paragraph{Comparison to the closed-form case.} We start by assessing the quality of the approximation encapsulated in Theorem \ref{thm:convergence_eff_dim} in a scenario where the eigenvalues can be theoretically bounded. 
When $d=1$, $s=1$, $\mathscr D = \frac{d}{dx}$, and $\Omega = [-\pi,\pi]$, one has \citep[Proposition 5.2]{doumeche2024physicsinformed} 
\[\frac{4}{(\lambda_n+\mu_n)(k+4)^2} \leqslant \sigma^\downarrow_k(C\mathscr O_nC)\leqslant \frac{4}{(\lambda_n+\mu_n)(k-2)^2}.\] 
This shows that $\log \sigma^\downarrow_k(C\mathscr O_nC) \sim_{k\to \infty} - 2\log(k)$.
Figure \ref{fig:1d_spectrum} (Left) represents the eigenvalues of $C_mM_m^{-1}C_m$ in decreasing order, for increasing values of $m$, with $\lambda_n = 0.01$ and $\mu_n = 1$. 
For any fixed $m$, two distinct regimes can be clearly distinguished: initially, the eigenvalues decrease linearly on a $\log-\log$ scale and align with the theoretical values of $- 2\log(k)$.
Afterward, the eigenvalues suddenly drop to zero. As $m$ increases, the spectrum progressively approaches the theoretical bound.

In Appendix~\ref{sec:eff_dim_m}, we show that $m = 10^{2}$ Fourier modes are sufficient to accurately approximate the effective dimension when $n \leqslant 10^4$. It is evident from Figure \ref{fig:1d_spectrum} (Right) that the effective dimension exhibits a sub-linear behavior in the $\log-\log$ scale, experimentally confirming the findings of \citet{doumeche2024physicsinformed}, which show that $\mathcal{N}(\frac{\log(n)}{n}, \frac{1}{\log(n)}) = o_{n\to\infty}(n^\gamma)$ for all $\gamma > 0$. 
So, plugging this into \eqref{eq:err_l2} with $\lambda_n = n^{-1}\log(n)$ and $\mu_n = \log(n)^{-1}$ leads to
\begin{equation*}
        \mathbb{E}\int_{[-L,L]} |\hat f_n-f^\star|^2 d{\mathbb P}_X = (\|f^\star\|_{H^1(\Omega)}^2 + \sigma^2 + M^2)O_n \big( n^{-1} \log^3(n)\big) 
\end{equation*}
when $\mathscr D(f^\star) = 0$, i.e., when the modeling is perfect.
The Sobolev minimax rate on $H^1(\Omega)$ is $n^{-2/3}$, whereas the experimental bound in this context gives a rate of $n^{-1}$. This indicates that when the target $f^{\star}$ satisfies the underlying PDE, the gain in terms of speed from incorporating the physics into the learning problem is  $n^{-1/3}$.
\begin{figure}
    \centering
    \includegraphics[scale=0.5]{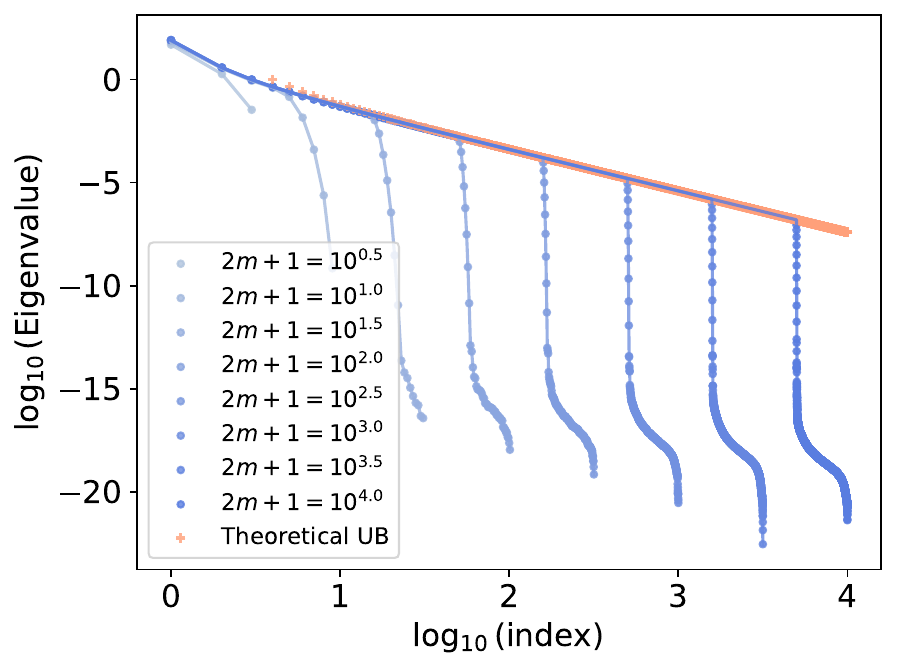}
    \includegraphics[scale=0.5]{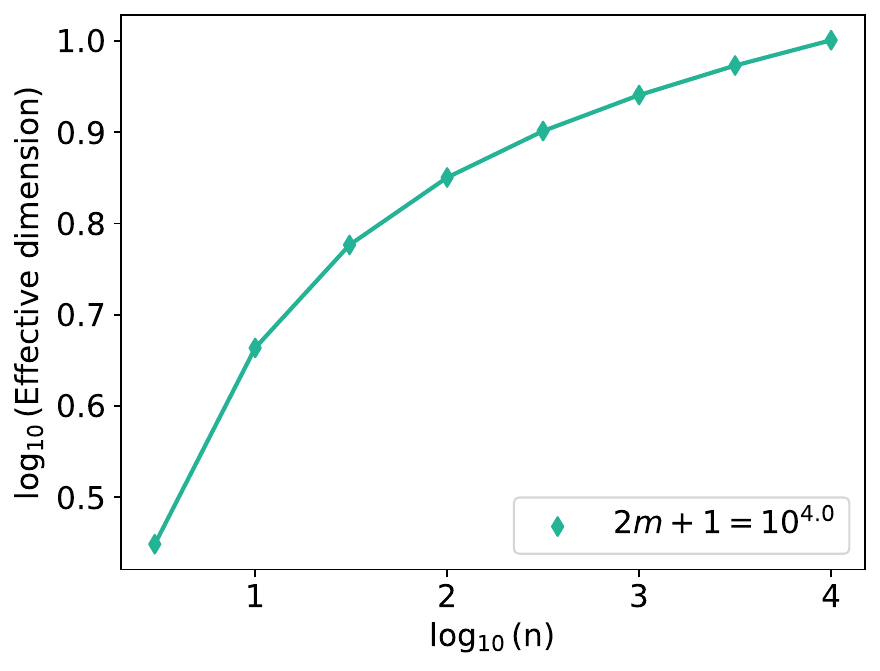}
    \caption{The case of $\mathscr D = \frac{d}{dx}$. \textbf{Left:} Spectrum of $C_mM_m^{-1}C_m$. \textbf{Right:} Estimation of the effective dimension $n \mapsto \mathcal N(\frac{\log(n)}{n}, \frac{1}{\log(n)})$. 
    }
    \label{fig:1d_spectrum}
\end{figure}

\paragraph{Harmonic oscillator equation.} Here, we follow up on the example of Subsection~\ref{sec:hybrid_mod}, as presented in Figures~\ref{fig:ressort_least_square} and \ref{fig:perf_ressort_least_square}. 
Thus, we set $d=1$, $s=2$, $\mathscr D(u) = \frac{d^2}{dx^2}u+\frac{d}{dx}u+u$, and $\Omega=[-\pi,\pi]$. Recall that in this perfect modeling experiment, we observed a parametric convergence rate of $n^{-1}$, which is not surprising since the regression problem essentially involves learning the two parameters $a_1$ and $a_2$. 
Figure \ref{fig:oscillator_spectrum} (Left) shows the eigenvalues of $C_m M_M^{-1} C_m$, while Figure \ref{fig:oscillator_spectrum} (Right) shows the effective dimension as a function of $n$. Similarly to the previous closed-form case, we observe that $\mathcal{N}(\frac{\log(n)}{n}, \frac{1}{\log(n)}) = o_{n\to\infty}(n^\gamma)$ for all $\gamma > 0$. 
The same argument as in the paragraph above shows that this results in a parametric convergence rate, provided $\mathscr D(f^\star) = 0$. 
\begin{figure}
    \centering
    \includegraphics[scale=0.5]{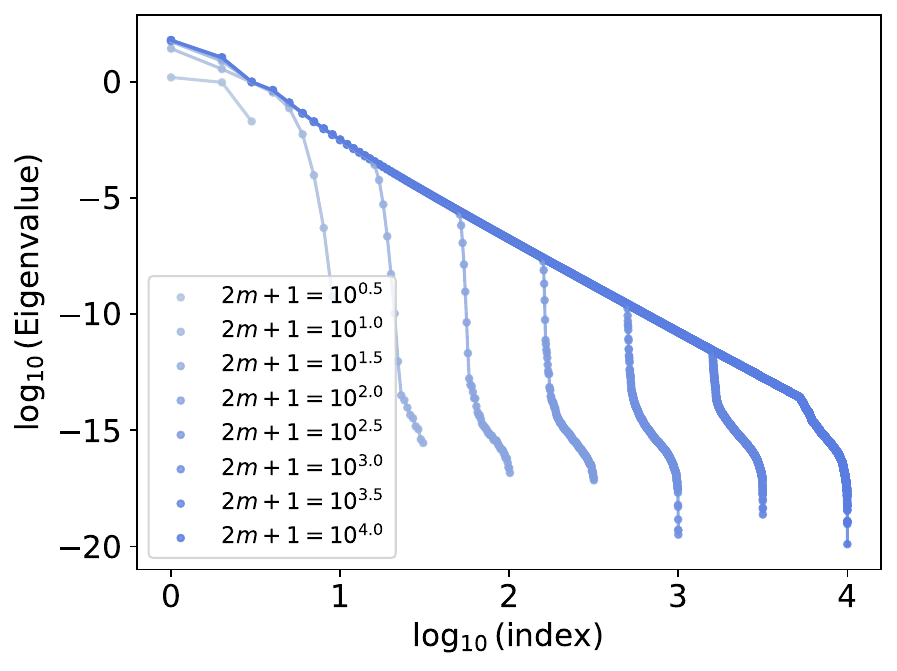}
    \includegraphics[scale=0.5]{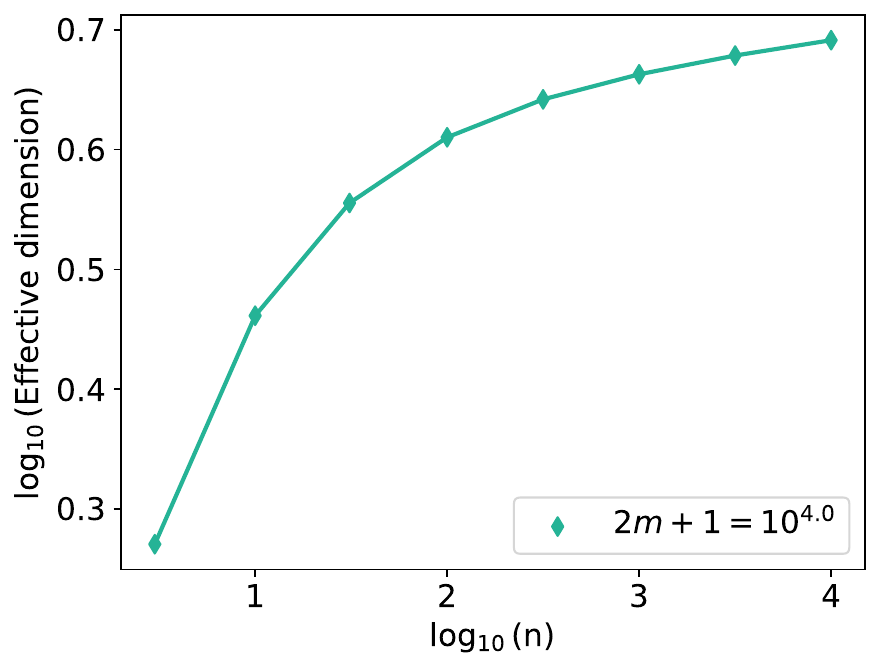}
    \caption{Harmonic oscillator. \textbf{Left:} Spectrum of $C_mM_m^{-1}C_m$. \textbf{Right:} Estimation of the effective dimension $n \mapsto \mathcal N(\frac{\log(n)}{n}, \frac{1}{\log(n)})$.}
    \label{fig:oscillator_spectrum}
\end{figure}

\paragraph{Heat equation on the disk.} Let us now consider the one-dimensional heat equation $\mathscr D = \frac{\partial}{\partial x}-\frac{\partial^2}{\partial y^2}$, with $d=2$, $s=2$, and the disk  $\Omega=\{x\in \mathbb R^2, \|x\|_2\leq\pi\}$. Since the heat equation is known to have $C^\infty$ solutions with bounded energy \citep[see, e.g.,][Chapter~2.3, Theorem~8]{evans2010partial}, we expect the convergence rate to match that of $H^{\infty}(\Omega)$, which corresponds to the parametric rate of $n^{-1}$.
\begin{figure}
    \centering
    \includegraphics[scale=0.5]{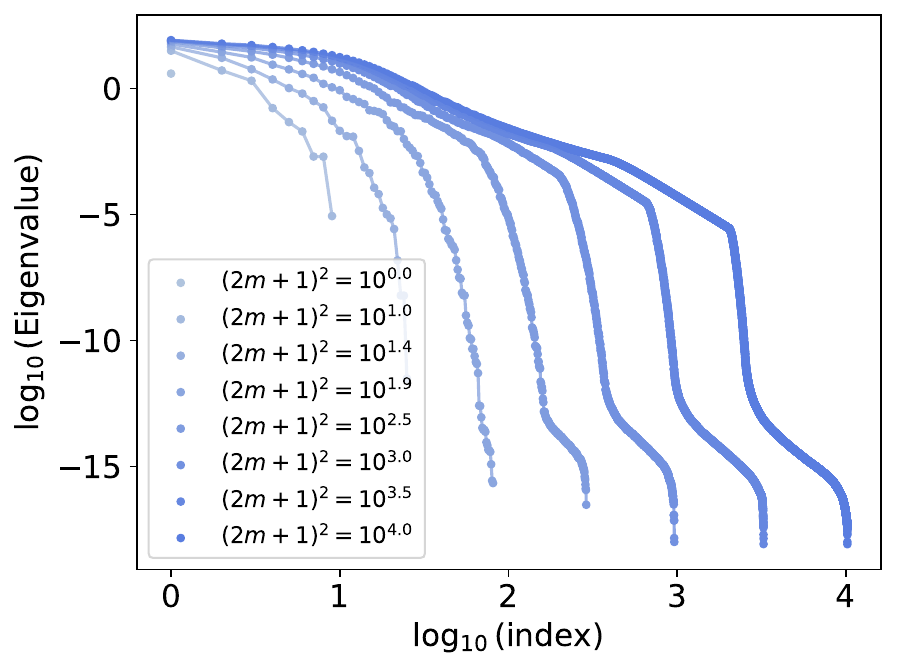}
    \includegraphics[scale=0.5]{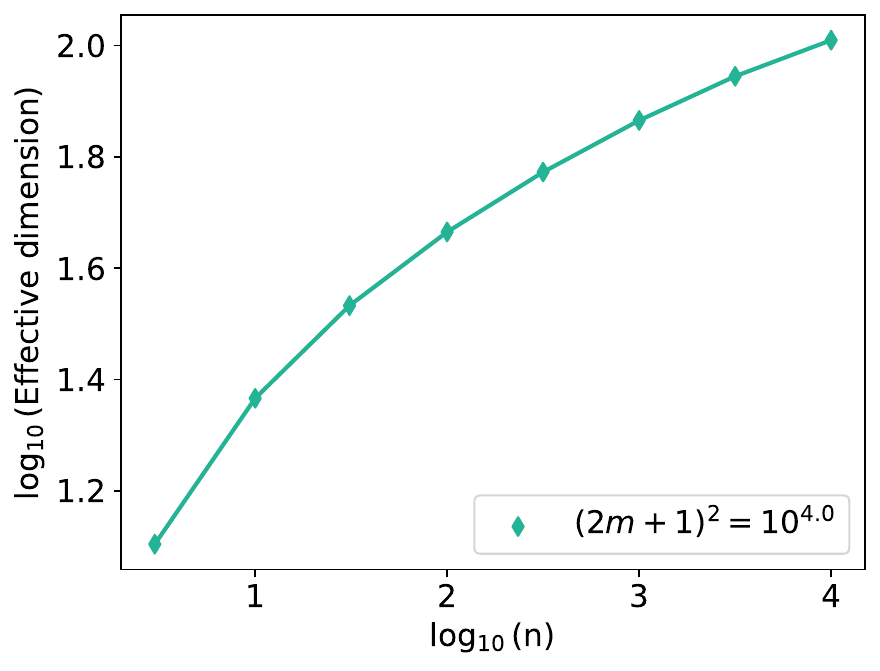}
    \caption{Heat equation. \textbf{Left:} Spectrum of $C_mM_m^{-1}C_m$. \textbf{Right:} Estimation of the effective dimension $n \mapsto \mathcal N(\frac{\log(n)}{n}, \frac{1}{\log(n)})$.}
    \label{fig:heat_spectrum}
\end{figure}
Once again, we observe $\mathcal{N}(\frac{\log(n)}{n}, \frac{1}{\log(n)}) = o_{n\to\infty}(n^\gamma)$ for all $\gamma > 0$, and thus an improvement over the $n^{-2/3}$-Sobolev minimax rate on $H^2(\Omega)$ when $\mathscr D(f^\star) = 0$.

\paragraph{Quantifying the impact of physics.} The three examples above show how incorporating physics can enhance the learning process by reducing the effective dimension, leading to a faster convergence rate. In all cases, the rate becomes parametric due to the PDE, achieving the fastest possible speed, as predicted by the central limit theorem. Our package can be directly applied to any linear PDE with constant coefficients to compute the effective convergence rate given a scaling of $\lambda_n$ and~$\mu_n$. By identifying  the optimal convergence rate, this approach can assist in determining the best parameters $\lambda_n$ and $\mu_n$ for use in other PIML techniques, such as PINNs.

\section{PDE solving: Mitigating the difficulties of PINNs with PIKL}
\label{sec:PDE_solving}
It turns out that our PIKL algorithm can be effectively used as a PDE solver. In this scenario, there is no noise (i.e., $\varepsilon = 0$), no modeling error (i.e., $\mathscr{D}(f^\star) = 0$), and the data consist of samples of boundary and initial conditions, as is typical for PINNs. Assume for example that the objective is to solve the Laplacian equation $\Delta(f^\star) = 0$ on a domain $\Omega \subseteq [-1,1]^2$ with the Dirichlet boundary condition $f^\star|_{\partial \Omega} = g$, where $g$ is a known function. Then this problem can be addressed by implementing the PIKL estimator, which minimizes the risk $\bar R_n(f) = \frac{1}{n}\sum_{i=1}^n |f(X_i) - Y_i|^2 + \lambda_n \|f\|^2_{H^2_{\mathrm{per}}([-1,1]^2)} + \mu_n \|\Delta(f)\|^2_{L^2(\Omega)}$, where the $X_i$ are uniformly sampled on $\partial \Omega$ and $Y_i = g(X_i)$.
Of course, this example focuses on Dirichlet boundary conditions, but PIKL is a highly flexible framework that can incorporate a wide variety of boundary conditions, such as periodic and Neumann boundary conditions, as the next two examples will illustrate.

\paragraph{Comparison with PINNs for the convection equation.} 
To begin, we compare the performance of our PIKL algorithm with the PINN approach developed by \citet{krishnapriyan2021characterizing} for solving the one-dimensional convection equation $\mathscr{D}(f) = \partial_{t} f + \beta \partial_{x} f$ on the domain $\Omega = [0,1]\times [0,2\pi]$. The problem is subject to the following periodic boundary conditions:
\[
\left\{
\begin{array}{l}
\forall x \in [0,1], \quad f(0, x) = \sin(x),\\
\forall t \in [0,1], \quad f(t, 0) = f(t, 2\pi) = 0.
\end{array}
\right.
\] 

The solution of this PDE is given by $f^\star(t,x) = \sin(x - \beta t)$. \citet{krishnapriyan2021characterizing} show that for high values of $\beta$, PINNs struggle to solve the PDE effectively. To address this challenge, we train our PIML kernel method using $n = 100$ data points and $1681$ Fourier modes (i.e., $m = 20$). The training data set $(X_i, Y_i)_{1 \leqslant i \leqslant n}$ is constructed such that $X_i = (0, U_i)$ and $Y_i = \sin(U_i)$, where $(U_i)_{1 \leqslant i \leqslant n}$ are i.i.d.~uniform random variables. To enforce the periodic boundary conditions,  we center $\Omega$ at $\tilde \Omega = \Omega-(0.5, \pi)$, extend it to $[-1, 1]\times[-\pi, \pi]$, and consider $\tilde H_m = \mathrm{Span}((t,x)\mapsto e^{i (\frac{\pi}{2}k_1 t + k_2 x)})_{\|k\|_\infty \leq m}$. Noting that for all $(j_1, k_1), (j_2, k_2)\in \mathbb Z^2$,
\[
\int_{[-1, 1]\times[-\pi, \pi]}e^{i (\frac{\pi}{2}(k_1-j_1) t + (k_2 - j_2) x)}dx = \frac{\sin(\pi (k_1-j_1)/2)}{\pi } \delta_{k_2,j_2},
\]
we let the matrix $(M_m)_{j,k}$ be as follows:
\begin{equation*}
    (M_m)_{j,k} = \lambda_n \Big(1+\frac{\|k\|_2^2}{(2L)^2}\Big)\delta_{j,k}+\mu_n \frac{P(j)\bar P(k)}{(4L)^2}\frac{\sin(\pi (k_1-j_1)/2)}{\pi } \delta_{k_2,j_2},
\end{equation*}
where $P$ is the polynomial associated with the operator $\mathscr D$.  Notice that, although $f^\star$ is a sinusoidal function, the frequency vector of $f^\star$ is $(-\beta, 1)$, which does not belong to $\frac{\pi}{2}\mathbb Z \oplus \mathbb Z$. As a result, $f^\star$ does not lie in $\tilde H_m$ for any $m$.

Table~\ref{table_perf_PDE_convection} compares the performance of various PIML methods using a sample of $n = 100$ initial condition points. The performance of an estimator $\hat{f}_n$ on a test set ($\mathrm{Test}$) is evaluated based on the $L^2$ relative error $(\sum_{x \in \mathrm{Test}} \|\hat f_n(x)-f^\star(x)\|_2^2/\sum_{y \in \mathrm{Test}} \|f^\star(y)\|_2^2)^{1/2}$.
Standard deviations are computed across 10 trials. The results show that the PIML kernel estimator clearly outperforms PINNs in terms of accuracy. 

\begin{table}[ht]
\centering
\caption{\textbf{$L^2$ relative error of the kernel method in solving the wave equation.} } 

\begin{tabular*}{\textwidth}{@{\extracolsep{\fill}}lccc}
  \toprule
  & Vanilla PINNs$^{\diamond}$ & Curriculum-trained PINNs$^{\diamond}$ & PIKL estimator\\
  \midrule
  \textit{$\beta = 20$ } & $7.50 \times 10^{-1}$&   $9.84 \times 10^{-3}$& $\mathbf{(1.56{\scriptstyle \pm 3.46}) \times 10^{-8}}$\\
  \textit{$\beta = 30$} & $8.97 \times 10^{-1}$ & $2.02 \times 10^{-2}$ & $\mathbf{(0.91{\scriptstyle \pm 2.20})\times 10^{-7}}$\\
  \textit{$\beta = 40$} & $9.61 \times 10^{-1}$ & $5.33 \times 10^{-2}$ & $\mathbf{(7.31{\scriptstyle \pm 6.44}) \times 10^{-9}}$\\
   \bottomrule
\end{tabular*}

\flushleft 
\textit{$^\diamond$ \citet[][Table 1]{krishnapriyan2021characterizing}}

\label{table_perf_PDE_convection}
\end{table}

\paragraph{Comparison with PINNs for the 1d-wave equation.} The performance of the PIKL algorithm is compared to the PINN methodology of \citet[Section 7.3]{wang2022when} for solving the one-dimensional wave equation $\mathscr{D}(f) = \partial^2_{t,t} f - 4 \partial_{x,x}^2 f$ on the square domain $[0,1]^2$, with the following boundary conditions: 
\[
\left\{
\begin{array}{l}
\forall x \in [0,1], \quad f(0, x) = \sin(\pi x) + \sin(4\pi x)/2,\\
\forall x \in [0,1], \quad \partial_tf(0, x) = 0, \\
\forall t \in [0,1], \quad f(t, 0) = f(t, 1) = 0.
\end{array} 
\right.
\]

The solution of the PDE is $f^\star(t,x) = \sin(\pi x)\cos(2\pi t) + \sin(4\pi x)\cos(8\pi t)/2$. This solution serves as an interesting benchmark since $f^\star$ exhibits significant variations, with $\|\partial_t f^\star\|_2^2/ \|f^\star \|_2^2 = 16 \pi^2$ (Figure~\ref{fig:wave_comp}, Left). Meanwhile, PINNs are known to have a spectral bias toward low frequencies \citep[e.g.,][]{Deshpande2022investigations, wang2022physics}. The optimization of the PINNs in \citet{wang2022physics} is carried out using stochastic gradient descent with $80,000$ steps, each drawing $300$ points at random, resulting in a sample size of $n = 2.4 \times 10^6$. The architecture of the PINNs these authors employ is a dense neural network with $\tanh$ activation functions and layers of sizes $(2, 500, 500, 500, 1)$, resulting in $m = (2 \times 500 + 500) + 2 \times (500 \times 500 + 500) + (500 \times 1 + 1) = 503,001$ parameters.
The training time for Vanilla PINNs is 7 minutes on an Nvidia L4 GPU (24 GB of RAM, 30.3 teraFLOPs for Float32). We obtain an $L^2$ relative error of $4.21\times 10^{-1}$, which is consistent with the results of \citet{wang2022when}, who report a $L^2$ relative error of $4.52\times 10^{-1}$. Figure~\ref{fig:wave_comp} (Middle) shows the Vanilla PINNs. 

We train our PIKL method using $n = 10^5$ data points and $1681$ Fourier modes (i.e., $m = 20$). Let $(U_i)_{1 \leqslant i \leqslant n}$ be i.i.d.~random variables uniformly distributed on 
$[0,1]$. The training data set $(X_i, Y_i)_{1 \leqslant i \leqslant n}$ is constructed such that
\begin{itemize}
    \item if $\ 1\leqslant i\leqslant \lfloor n/4\rfloor$, then $X_i = (0, U_i)$  and $Y_i = \sin(\pi U_i) + \sin(4\pi U_i)/2$,
    \item if $ \lfloor n/4\rfloor+1\leqslant i\leqslant 2\lfloor n/4\rfloor$, then $X_i = (U_i, 0)$ and $Y_i = 0$,
    \item if $ 2\lfloor n/4\rfloor+1\leqslant i\leqslant 3\lfloor n/4\rfloor$, then $X_i = (U_i, 1)$ and $Y_i = 0$,
    \item if $ 3\lfloor n/4\rfloor+1\leqslant i\leqslant  n$, then  $X_i = (1/n, U_i)$ and 
    \begin{align*}
        Y_i &= f(0, U_i) + \frac{1}{2n^2}\partial^2_{t,t}f(0, U_i) = f(0, U_i) + \frac{2}{n^2}\partial^2_{x,x}f(0, U_i) \\
        &= \Big(1-\frac{2\pi^2}{n^2}\Big)\sin(\pi U_i) + \Big(\frac{1}{2}-\frac{16\pi^2}{n^2}\Big)\sin(4\pi U_i).
    \end{align*}
\end{itemize}
The final requirement enforces the initial condition $\partial_t f = 0$ in a manner similar to that of a second-order numerical scheme.

Table~\ref{table_perf_PDE} compares the performance of the PINN approach from \citet{wang2022when} with the PIKL estimator. Across $10$ trials, the PIKL method achieves an $L^2$ relative error of $(8.70 \pm 0.08) \times 10^{-4}$, which is $50\%$ better than the performance of the PINNs. This demonstrates that the kernel approach is more accurate, requiring fewer data points and parameters than the PINNs. The training time for the PIKL estimator is 6 seconds on an Nvidia L4 GPU. Thus, the PIKL estimator can be computed 70 times faster than the Vanilla PINNs. Figure~\ref{fig:wave_comp} (Right) shows the PIKL estimator. 
Note that in this case, the solution $f^\star$ can be represented by a sum of complex exponential functions ($f^\star\in H_{16}$), which could have biased the result in favor of the PIKL estimator by canceling its approximation error. However, the results remain unchanged when altering the frequencies in $H_m$ (e.g., taking $L=0.55$ in \eqref{eq:Mm} instead of $L=0.5$ yields an $L^2$ relative error of $(9.6\pm 0.3)\times 10^{-4} $).
\begin{table}[ht]
\centering
\caption{\textbf{Performance of PINN/PIKL methods for solving the wave equation on $\Omega=[0,1]^2$}}. 

\begin{tabular*}{\textwidth}{@{\extracolsep{\fill}}lccc}
  \toprule
  & Vanilla PINNs$^{\diamond}$ & NTK-optimized PINNs$^{\diamond}$ & PIKL estimator\\
  \midrule
  \textit{$L^2$ relative error} & $4.52 \times 10^{-1}$&   $1.73 \times 10^{-3}$& $\mathbf{(8.70 {\scriptstyle \pm 0.08})  \times 10^{-4}}$\\
  \textit{Training data (n)} & $2.4 \times 10^{6}$ & $2.4 \times 10^{6}$ & $\mathbf{10^5}$\\
  \textit{Number of parameters} & $5.03 \times 10^5$ & $5.03 \times 10^5$ & $\mathbf{1.68 \times 10^3}$\\
   \bottomrule
\end{tabular*}

\flushleft 
\textit{$^\diamond$ \citet[][Figure 6]{wang2022when}}

\label{table_perf_PDE}
\end{table}
\begin{figure}
    \centering
    \includegraphics[width=0.3\textwidth]{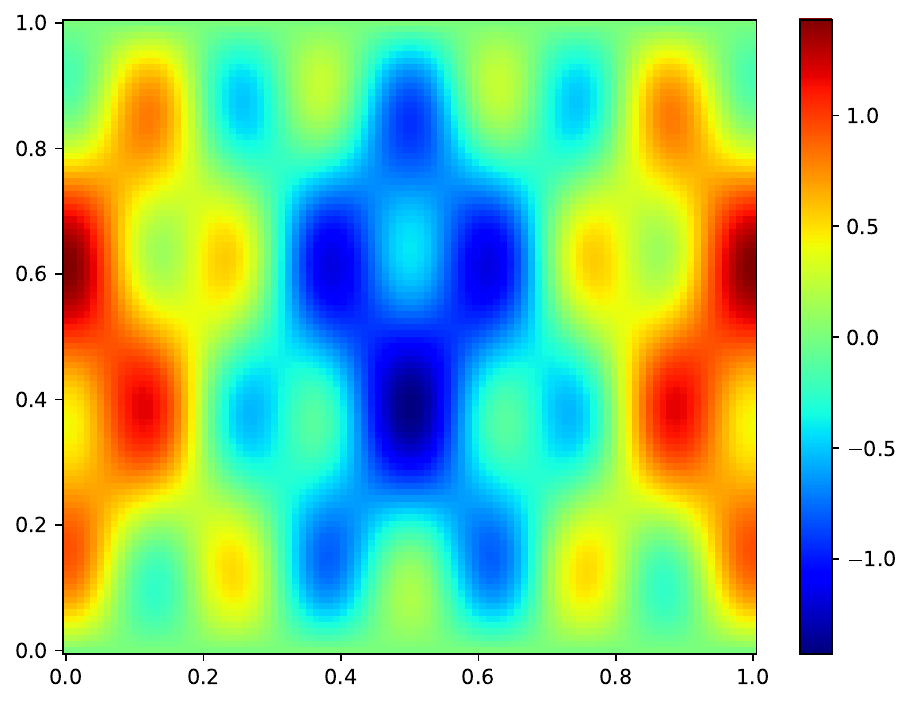}
    \includegraphics[width=0.3\textwidth]{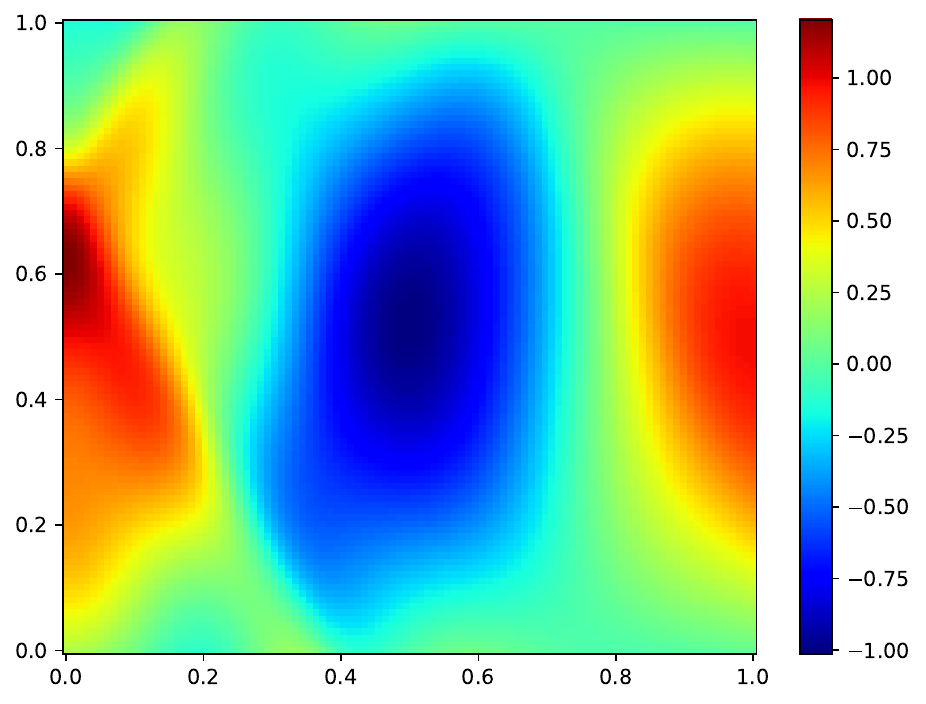}
    \includegraphics[width=0.3\textwidth, trim={15.6cm 0.65cm 0 0},clip]{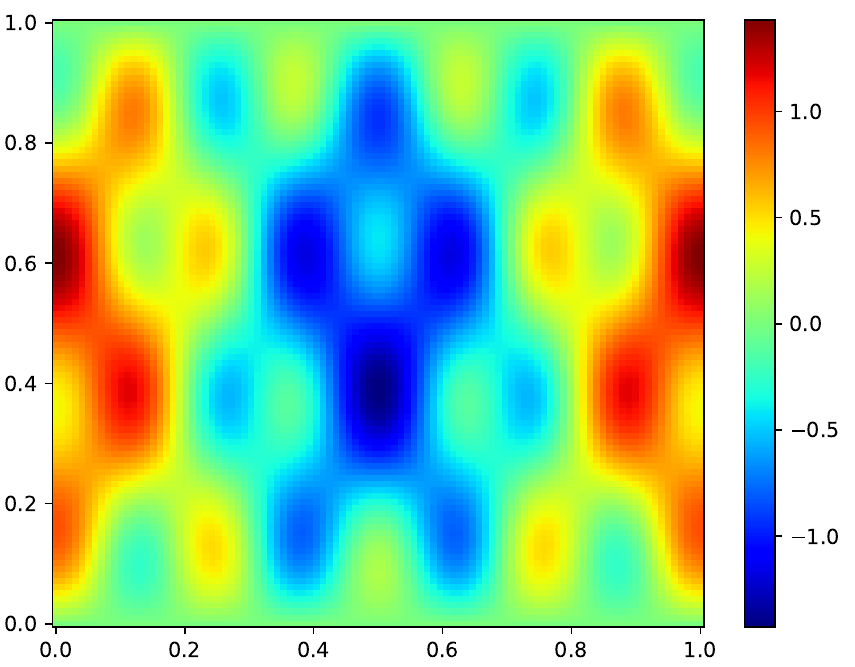}
    \caption{\textbf{Left:} ground truth solution $f^\star$ to the wave equation \citep[taken from][Figure~6]{wang2022when}. \textbf{Middle:} Vanilla PINNs from \citet{wang2022when}. \textbf{Right:} PIKL estimator. }
    \label{fig:wave_comp}
\end{figure}

\section{Conclusion and future directions}
\paragraph{Conclusion.} In this article, we developed an efficient algorithm to solve the PIML hybrid problem \eqref{eq:estimator_sob_0}. The PIKL estimator can be computed exactly through matrix inversion and possesses strong theoretical properties. Specifically, we demonstrated how to estimate its convergence rate based on the PDE prior $\mathscr{D}$. Moreover, through various examples, we showed that it outperforms PINNs in terms of performance, stability, and training time in certain PDE-solving tasks where PINNs struggle to escape local minima during optimization.
Future work could focus on comparing PIKL with the implementation of RFF and exploring its performance against PINNs in the case of PDEs with non-constant coefficients. Another avenue for future research is to assess the effectiveness of the kernel approach compared to traditional PDE solvers, as discussed below.

\paragraph{Comparison with traditional PDE solvers.} PIML is a promising framework for solving PDEs, particularly due to its adaptability to domains $\Omega$ with complex geometries, where most traditional PDE solvers tend to be highly domain-dependent. However, its comparative performance against traditional PDE solvers remains unclear in scenarios where both approaches can be easily implemented. The meta-analysis by \citet{mcgreivy2024weakbaselinesreportingbiases} indicates that, in some cases, PINNs may be faster than traditional PDE solvers, although they are often less accurate. In our study, solving the wave equation on a simple square domain represents a setting where traditional numerical methods are straightforward to implement and are known to perform well. Table~\ref{table_perf_PDE_solvers} summarizes the performance of classical techniques, including the explicit Euler, Runge-Kutta 4 (RK4), and Crank-Nicolson (CN) schemes (see Appendix~\ref{sec:numerical} for a brief presentation of these methods). These methods clearly outperform both PINNs and the PIKL algorithm, even with fewer data points.
\begin{table}[ht]
\centering
\caption{\textbf{Performance of traditional PDE solvers for the wave equation on $\Omega=[0,1]^2$}.} 

\begin{tabular*}{\textwidth}{@{\extracolsep{\fill}}lccc}
  \toprule
  & Euler explicit  & RK4 & CN\\
  \midrule
  \textit{$L^2$ relative error} & $\mathbf{3.8 \times 10^{-6}}$&   $6.8 \times 10^{-6}$& $ 5.6 \times 10^{-3}$\\
  \textit{Training data (n)} & $10^{4}$ & $10^{4}$ & $10^4$\\
   \bottomrule
\end{tabular*}
\label{table_perf_PDE_solvers}
\end{table}

However, a more relevant setting for comparing the performance of these methods arises when noise is introduced into the boundary conditions. This situation is common, for instance, when the initial condition of the wave is measured by a noisy sensor. Such a setting aligns with hybrid modeling, where $\varepsilon \neq 0$, but there is no modeling error (i.e., $\mathscr{D}(f^\star) = 0$). Table~\ref{table_perf_PDE_solvers_noisy} compares the performance of all methods with the same number $n$ of training samples, with Gaussian noise of variance of $10^{-2}$. In this case, the PIKL estimator outperforms all other approaches.
\begin{table}[ht]
\centering
\caption{\textbf{Performance for the wave equation with noisy boundary conditions}.} 
\begin{tabular*}{\textwidth}{@{\extracolsep{\fill}}lcccccc}
  \toprule
  & PINNs & Euler explicit & RK4 & CN  & PIKL estimator\\
  \midrule
  \textit{$L^2$ relative error} & $4.61 \times 10^{-1}$& $1.25 \times 10^{-1}$&   $6.05\times 10^{-2}$ & $2.01  \times 10^{-2}$&    $\mathbf{1.87 \times 10^{-2}}$\\
  \textit{Training data (n)} & $2.4\times10^{6}$ & $4\times10^{4}$ & $4\times 10^{4}$ & $4\times 10^{4}$& $4\times 10^4$\\
   \bottomrule
\end{tabular*}
\label{table_perf_PDE_solvers_noisy}
\end{table}
Such PDEs with noisy boundary conditions are special cases of the hybrid modeling framework, where the data is located on the boundary of the domain. This situation arises, for example, in \citet{cai2021physics} which models the temperature in the core of a nuclear reactor.

\newpage

\appendix
\section{Spectral methods and PIKL}
\label{sec:discussion_spec}
The Fourier approximation on which the PIKL algorithm relies resembles usual spectral methods. Spectral methods are a class of numerical techniques used to solve PDEs by representing the solution as a sum of basis functions, typically trigonometric (Fourier series) or polynomial (Chebyshev or Legendre polynomials). These methods are particularly powerful for problems with smooth solutions and periodic or well-behaved boundary conditions \citep[e.g.,][]{canuto2007spectral}. However the basis functions used in spectral and pseudo-spectral methods must be
specifically tailored to the domain $\Omega$, the differential operator $\mathscr D$, and the boundary conditions. 
This customization ensures that the method effectively captures the characteristics of the problem being solved. For example, the Fourier basis is unable to accurately reconstruct non-periodic functions on a square domain, leading to the Gibbs phenomenon at points of periodic discontinuity. A natural solution to this problem is to extend the solution of the PDE from the domain $\Omega$ to a simpler domain that admits a known spectral basis \citep[e.g.,][for Fourier basis extension]{Matthysen2016fast}. If the solution of the PDE on $\Omega$ can be extended to a solution of the same PDE on the extended domain, it becomes possible to apply a spectral method directly to the extended domain \citep[e.g.,][]{badea2001on, LUI2009541}. However, the PDE must satisfy certain regularity conditions (e.g., ellipticity), and there must be a method to implement the boundary conditions on $\partial \Omega$ instead of on the boundary of the extended domain.

In this article, we take a slightly different approach. Although we extend $\Omega \subseteq [-L, L]^d$ to $[-2L, 2L]^d$, we impose the PDE only on $\Omega$ and not on the entire extended domain $[-2L, 2L]^d$. Also, unlike spectral methods, we do not require that $\mathscr{D}(\hat{f}^{\mathrm{PIKL}}) = 0$. Instead, to ensure that the problem is well-posed, we regularize the PIML problem using the Sobolev norm of the periodic extension. This Tikhonov regularization is a conventional approach in kernel learning and is known to resemble spectral methods because it acts as a low-pass filter \citep[see, e.g.,][]{caponnetto2007optimal}. However, given a kernel, it is non-trivial to identify the basis of orthogonal functions that diagonalize it. The main contribution of this article is to establish an explicit connection between the Fourier basis and the PIML kernel, leading to surprisingly simple formulas for the kernel matrix $M_m$.

\section{Fundamentals of functional analysis on complex Hilbert spaces}
Let $L>0$ and $d\in \mathbb N^\star$. We define $L^2([-2L,2L]^d, \mathbb C)$ as the space of complex-valued functions $f$ on the hypercube $[-2L,2L]^d$ such that $\int_{[-2L,2L]^d} |f|^2 < \infty$. The real part of $f$ is denoted by $\Re(f)$, and the imaginary part by $\Im(f)$, such that $f = \Re(f) + i \Im(f)$. Throughout the appendix, for the sake of clarity, we use the dot symbol $\cdot$ to represent functions. For example, $\|\cdot\| $ denotes the function $x\mapsto \|x\|$, and $\langle\cdot, \cdot\rangle$ stands for the function $( x,y) \mapsto \langle x, y\rangle$.
\begin{defi}[$L^2$-space and $\|\cdot\|_2$-norm]
    The separable Hilbert space $L^2([-2L,2L]^d, \mathbb C)$ is associated with the inner product $\langle f,g\rangle = \int_{[-2L,2L]^d} f\bar g$ and the norm $\|f\|_2^2 = \int_{[-2L,2L]^d} |f|^2$.
\end{defi}
Let $s\in\mathbb N$. 
\begin{defi}[Periodic Sobolev spaces]
    The periodic Sobolev space $H^s_{\mathrm{per}}([-2L,2L]^d, \mathbb R)$ is the space of real functions $f(x) = \sum_{k\in\mathbb Z^d} z_k \exp(\frac{i\pi}{2L}\langle k,x\rangle)$ such that the Fourier coefficients $z_k$ satisfy $\sum_k |z_k|^2(1+\|k\|^{2s}) <\infty$. The corresponding complex periodic Sobolev space is defined by \[H^s_{\mathrm{per}}([-2L,2L]^d, \mathbb C) = H^s_{\mathrm{per}}([-2L,2L]^d, \mathbb R) \oplus i\; H^s_{\mathrm{per}}([-2L,2L]^d, \mathbb R).\] It is the space of complex-valued functions $f(x) = \sum_{k\in\mathbb Z^d} z_k \exp(\frac{i\pi}{2L}\langle k,x\rangle)$ such that $\sum_k |z_k|^2(1+\|k\|^{2s}) <\infty$.
\end{defi}
We recall that, given two Hilbert spaces $\mathcal H_1$ and $\mathcal H_2$, an operator is a linear function from $\mathcal H_1$ to $\mathcal H_2$.
\begin{defi}[Operator norm]\citep[e.g.,][Section 2.6]{brezis2010functional}
\label{defi:op_norm}
    Let $\mathscr O: L^2([-2L, 2L]^d, \mathbb C)\to L^2([-2L, 2L]^d, \mathbb C)$ be an operator. Its operator norm $|||\mathscr O|||_2$ is defined by
    \[|||\mathscr O|||_2 =\underset{\|g\|_2 = 1}{\sup_{g \in L^2([-2L, 2L]^d, \mathbb C)}} \|\mathscr O g\|_2 = \underset{g \neq 0}{\sup_{g \in L^2([-2L, 2L]^d, \mathbb C)}} \|g\|_2^{-1}\|\mathscr O g\|_2. \]
\end{defi}
The operator norm is sub-multiplicative, i.e., $|||\mathscr O_1 \circ \mathscr O_2|||_2 \leqslant |||\mathscr O_1|||_2 \times |||\mathscr O_2|||_2$. 
\begin{defi}[Adjoint]
    Let $(\mathcal H, \langle\cdot, \cdot\rangle_{\mathcal H})$ be an Hilbert space and $\mathscr O:\mathcal H \to \mathcal H$ be an operator. The adjoint $\mathscr O^\star$ of $\mathscr O$ is the unique operator such that $\forall f,g \in \mathcal H$, $\langle f, \mathscr O g\rangle_{\mathcal H} = \langle \mathscr O^\star f,  g\rangle_{\mathcal H}$.
\end{defi}
If $\mathcal H = \mathbb R^d$ with the canonical scalar product, then $\mathscr O^\star$ is the $d\times d$ matrix $\mathscr O^\star = \mathscr O^T$. If $\mathcal H = \mathbb C^d$ with the canonical sesquilinear inner product, then $\mathscr O^\star$ is the $d\times d$ matrix $\mathscr O^\star = \bar{\mathscr O}^T$.
\begin{defi}[Hermitian operator]
    Let $\mathcal H$ be an Hilbert space and $\mathscr O:\mathcal H \to \mathcal H$ be an operator. The operator $\mathscr O$ is said to be Hermitian if $\mathscr O = \mathscr O^\star$.
\end{defi}
\begin{theorem}[Spectral theorem]\citep[e.g.][Theorems 12.29 and 12.30]{rudin1991functional}
    Let $\mathscr O$ be a positive Hermitian compact operator. Then $\mathscr O$ is diagonalizable on an Hilbert basis with positive eigenvalues that tend to zero. We denote its eigenvalues, ordered in decreasing order, by $\sigma(\mathscr O) = (\sigma_k^\downarrow(\mathscr O))_{k\in \mathbb N^\star}$.
    \label{thm:spectral}
\end{theorem}
We emphasize that, given an invertible positive self-adjoint compact operator $\mathscr O$ and its inverse $\mathscr O^{-1}$, the eigenvalues of $\mathscr O^{-1}$ can also be ordered in increasing order, i.e., 
\begin{equation}
    \sigma(\mathscr O^{-1}) = (\sigma_k^\uparrow(\mathscr O^{-1}))_{k\in \mathbb N^\star} = (\sigma_k^\downarrow(\mathscr O)^{-1})_{k\in \mathbb N^\star}. \label{eq:order}
\end{equation}
\begin{theorem}[Courant-Fischer minmax theorem ]\citep[][Problem 37]{brezis2010functional}
    Let $\mathscr O: \mathcal H \to \mathcal H$ be a positive Hermitian compact operator. Then
    \[\sigma_k^\downarrow(\mathscr O) = \underset{\dim H = k}{\max_{H\subseteq \mathcal H}} \;  \underset{\|g\|_2=1}{\min_{g\in H }} \langle g, \mathscr O g\rangle_{\mathcal H}.\]
    If $\mathscr O$ is injective, then
\[\sigma_k^\uparrow(\mathscr O^{-1}) = \underset{\dim H = k}{\min_{H\subseteq \mathscr O(\mathcal H)}} \;  \underset{\|g\|_2=1}{\max_{g\in H }} \langle g, \mathscr O^{-1} g\rangle_{\mathcal H}.\]
\label{thm:minmax}
\end{theorem}
Interestingly, if $\mathscr O$ is a positive Hermitian compact operator, then Theorem~\ref{thm:minmax} shows that $|||\mathscr O|||_2$ equals its largest eigenvalue. 
\begin{defi}[Orthogonal projection on $H_m$]
    We let $\Pi_m: H^s_{\mathrm{per}}([-2L,2L]^d, \mathbb C) \to H_m$ be the orthogonal projection with respect to $\langle\cdot, \cdot\rangle$, i.e., for all $f \in H^s_{\mathrm{per}}([-2L,2L]^d, \mathbb C),$ \[ \Pi_m f(y) = \sum_{\|{k}\|_\infty \leqslant m}\Big(\frac{1}{(4L)^{d}}\int_{[-2L,2L]^d}\exp\Big(-\frac{i\pi}{2L}\langle {k}, x\rangle\Big)f(x)dx\Big)\exp\Big(\frac{i\pi}{2L}\langle {k}, y\rangle\Big).\]
    Note that $\Pi_m$ is Hermitian and that, for all $f \in L^2([-2L,2L]^d, \mathbb C)$, $\lim_{m\to\infty}\|f-\Pi_m f\|_2 = 0$.
\end{defi}

\section{Theoretical results for PIKL}
\label{app:Theory}
\subsection{Proof of Proposition \ref{prop:geom_cube}}
We have
\begin{align*}
    \frac{1}{(4L)^{d}}\int_{[-L,L]^d} e^{\frac{i \pi}{2L}\langle k, x\rangle}dx &= \frac{1}{(4L)^{d}}\prod_{j=1}^d\int_{[-L,L]} e^{\frac{i \pi}{2L} k_j x}dx = \prod_{j=1}^d\Big[\frac{1}{2i\pi} e^{\frac{i \pi}{2L} k_j x}\Big]_{x=-L}^L\\
    &= \prod_{j=1}^d\frac{e^{\frac{i \pi}{2} k_j }-e^{-\frac{i \pi}{2} k_j }}{2i\pi k_j} = \prod_{j=1}^d\frac{\sin(\frac{\pi}{2} k_j)}{\pi k_j}.
\end{align*}
The characteristic function of the Euclidean ball is computed in \citet[][Table 13.4]{bracewell}.
\subsection{Operations on characteristic functions}
\begin{prop}[Operations on characteristic functions]
\label{prop:op_char}
    Consider $d \in \mathbb N^\star$, $L >0$, and $\Omega \subseteq [-L,L]^d$.
    \begin{itemize}
        \item Let $a \in [-1,1]$. Then $a\cdot\Omega \subseteq [-L,L]^d$ and 
        \[F_{a\cdot\Omega}(k) = |a|^d\times F_{\Omega}(a\cdot k).\]
        \item Let $\tilde \Omega \subseteq [-L,L]^d$ be a domain such that $\Omega \cap \tilde \Omega = \emptyset$. Then $\Omega \sqcup \tilde \Omega \subseteq [-L,L]^d$ and
        \[F_{\Omega \sqcup \tilde \Omega}(k) = F_{\Omega}(k) + F_{\tilde \Omega}(k).\]
        \item Assume that $\Omega \subseteq [-L/2,L/2]^d$, and let $z \in \mathbb R^d$ be such that $\|z\|_\infty < L/2$. Then $\Omega + z \subseteq [-L,L]^d$ and \[F_{\Omega + z}(k) = F_{\Omega}(k) \times \exp\Big(\frac{i \pi}{2L} \langle k,z\rangle\Big).\]
        \item Assume that $\Omega = \Omega_1 \times \Omega_2$, where $\Omega_1\subseteq [-L,L]^{d_1}$, $\Omega_2 \subseteq [-L,L]^{d_2}$, and $d_1+d_2 = d$.  Then
        \[F_{\Omega}(k) = F_{\Omega_1}(k_1, \hdots, k_{d_1}) \times F_{\Omega_2}(k_{d_1+1}, \hdots, k_{d}).\]   
    \end{itemize}
    \end{prop}

\subsection{Operator extensions}
\begin{defi}[Projection on $\Omega$]
    We define the projection $C: L^2(\mathbb R^d, \mathbb C) \to L^2(\mathbb R^d, \mathbb C)$ on $\Omega$  by $C(f) = 1_\Omega f$.
\end{defi}

\begin{defi}[Operator extensions]
    The operators $C_m: H_m \to H_m$, $M_m: H_m \to H_m$, and $M_m^{-1}: H_m \to H_m$ can be extended to $L^2([-2L,2L]^d,  \mathbb C)$ by $C_m = \Pi_m C \Pi_m$, $M_m = \Pi_m M_m \Pi_m$, and $M_m^{-1} = \Pi_m M_m^{-1} \Pi_m$.
\end{defi}
From now on, we consider the extensions of these operators, allowing us to express equivalently \[|||M_m^{-1}|||_2 = \underset{\|g\|_2 = 1}{\sup_{g \in H_m}} \|M_m^{-1}g\|_2  = \underset{\|g\|_2 = 1}{\sup_{g \in L^2([-2L, 2L]^d)}} \|M_m^{-1}g\|_2.\]
It is important to note that the extended operator $M_m^{-1}$ is no longer the inverse of the extended operator $M_m$.

\begin{prop}[Compact operator extension]
    Let $\mathscr O$ be a positive Hermitian compact  operator on $L^2([-2L,2L]^d, \mathbb R)$. Then its unique extension $\tilde{\mathscr O}$ to $L^2([-2L,2L]^d, \mathbb C)$ is a positive Hermitian compact operator with the same real eigenfunctions and positive eigenvalues.
    \label{prop:extension}
\end{prop}
\begin{proof}
    Since $\tilde{\mathscr O}$ is $\mathbb C$-linear, we necessarily have $\tilde{\mathscr O} (f) = \mathscr O (\Re(f)) + i \mathscr O (\Im(f))$. Therefore, the extension is unique. Since $\mathscr O$ is compact, $\tilde{\mathscr O}$ is also compact. According to Theorem \ref{thm:spectral}, the operator $\mathscr O$ is diagonalizable in a Hermitian basis $(f_k)_{k\in \mathbb N^\star}$. Thus, for all $f \in L^2([-2L,2L], \mathbb R)$,
    \[\mathscr O(f) = \sum_{k\in\mathbb N^\star} \sigma_k^\downarrow(\mathscr O) \langle f, f_k\rangle_{L^2([-2L,2L], \mathbb R)}f_k.\] 
    Thus, for all $f\in L^2([-2L,2L], \mathbb C)$
    \begin{align*}
        \tilde{\mathscr O} (f) &= \sum_{k\in\mathbb N^\star} \sigma_k^\downarrow(\mathscr O) (\langle \Re(f), f_k\rangle_{L^2([-2L,2L], \mathbb R)}+ i\langle \Im(f), f_k\rangle_{L^2([-2L,2L], \mathbb R)})f_k\\
        &= \sum_{k\in\mathbb N^\star} \sigma_k^\downarrow(\mathscr O) \langle f, f_k\rangle_{L^2([-2L,2L], \mathbb C)}f_k.
    \end{align*}
    This formula shows that $\tilde{\mathscr O}$ is Hermitian and diagonalizable with the same real eigenfunctions and positive eigenvalues as $\mathscr O$.
\end{proof}
 Recall that $\mathscr O_n$ is the operator $\mathscr O_n = \lim_{m\to \infty} M_m^{-1}$, where the limit is taken in the sense of the operator norm \citep[see Proposition~B.2][]{doumeche2024physicsinformed}.  
\begin{defi}[Operator $M$]
    Proposition \ref{prop:extension} shows that the operator $\mathscr O_n$ can be extended to $L^2([-2L,2L], \mathbb C)$. We denote the extension of $\mathscr O_n$ by $M^{-1}$. 
\end{defi}
The uniqueness of the extension in Proposition \ref{prop:extension} implies that the extension of the operator $C\mathscr O_nC: L^2([-2L,2L]^d)\to L^2([-2L,2L]^d)$ to $\mathbb C$ is indeed $CM^{-1}C: L^2([-2L,2L]^d, \mathbb C)\to L^2([-2L,2L]^d, \mathbb C)$.
Proposition \ref{prop:extension} shows that $C\mathscr O_nC$ has the same eigenvalues as $CM^{-1}C$.

\subsection{Convergence of $M_m^{-1}$}
\begin{lem}[Bounding the spectrum of $M_m^{-1}$]
    Let $m \in \mathbb N^\star$. Then, for all $k\in \mathbb N^\star$, 
    \[\sigma_k^\downarrow(M_m^{-1}) \leqslant  \sigma_k^\downarrow(M^{-1}).\]
    \label{lem:bounding_spectrum}
\end{lem}
\begin{proof}
    Let $f \in H_m$. Then, \[\langle f, M_m f\rangle = \langle f, M f\rangle = \lambda_n \|f\|_{H^s_{\mathrm{per}}([-2L,2L]^d)}^2 + \mu_n \|\mathscr D(f)\|_{L^2(\Omega)}^2.\]
    Thus, using Theorem \ref{thm:minmax}, we deduce that 
    \begin{align*}
    \sigma_k^\uparrow(M_m) &= \underset{\dim H = k}{\min_{H\subseteq H_m}} \;  \underset{\|g\|_2=1}{\max_{g\in H }} \langle g, M_m g\rangle\\
    &= \underset{\dim H = k}{\min_{H\subseteq H_m}} \; \underset{\|g\|_2=1}{\max_{g\in H }} \langle g, M g\rangle\\
    &\geqslant  \underset{\dim H = k}{\min_{H\subseteq H^s_{\mathrm{per}}([-2L,2L]^d)}} \; \underset{\|g\|_2=1}{\max_{g\in H }} \langle g, M g\rangle\\
    &=\sigma_k^\uparrow(M).
    \end{align*} 
    From \eqref{eq:order}, we deduce that $\sigma_k^\downarrow(M_m^{-1}) \leqslant \sigma_k^\downarrow(M^{-1})$. 
\end{proof}

\begin{lem}[Spectral convergence of $M_m$]
    Let $m \in \mathbb N^\star$. Then, for all $k\in \mathbb N^\star$, one has
    \[\lim_{m\to\infty}\sigma_k^\downarrow(M_m^{-1}) = \sigma_k^\downarrow(M^{-1}).\]
    \label{lem:convergence_spectrum}
\end{lem}
\begin{proof}
    By continuity of the RKHS norm $f\mapsto \langle f, Mf\rangle$ on $H^s_{\mathrm{per}}([-2L,2L]^d, \mathbb C)$ \citep[][Proposition B.1]{doumeche2024physicsinformed}, we know that, for all function $f\in H^s_{\mathrm{per}}([-2L,2L]^d, \mathbb C)$, the quantity $ \lambda_n \|\Pi_m(f)\|_{H^s_{\mathrm{per}}([-2L,2L]^d, \mathbb C)}^2 + \mu_n \|\mathscr D(\Pi_m(f))\|_{L^2(\Omega, \mathbb C)}^2$ converges to $ \lambda_n \|f\|_{H^s_{\mathrm{per}}([-2L,2L]^d, \mathbb C)}^2 + \mu_n \|\mathscr D(f)\|_{L^2(\Omega, \mathbb C)}^2$ as $m$ goes to the infinity. Thus, 
    \begin{equation*}
    \forall f\in H^s_{\mathrm{per}}([-2L,2L]^d, \mathbb C),\quad \lim_{m\to \infty}\langle f, (M-M_m)f\rangle = 0.
    \end{equation*}

Next, consider $f_1, \hdots, f_k$ to be the eigenfunctions of $M$ associated with the ordered eigenvalue $\sigma_1^\uparrow(M), \hdots,\sigma_k^\uparrow(M)$. 
    Since, for any $1\leqslant j,\ell \leqslant k$, we have that $\lim_{m\to \infty}\langle f_j + f_\ell, M_m(f_j+f_\ell)\rangle = \langle f_j + f_\ell, M(f_j+f_\ell)\rangle$ and $\lim_{m\to \infty}\langle f_j + f_\ell, M_m(f_j+f_\ell)\rangle = \lim_{m\to \infty}\langle f_j, M_m f_j\rangle + \lim_{m\to \infty}\langle f_\ell, M_m f_\ell\rangle + 2\lim_{m\to \infty}\Re(\langle f_j, M_m f_\ell\rangle)$, we deduce that $\lim_{m\to \infty}\;\Re(\langle f_j, M_m f_\ell\rangle) = $ $\Re(\langle f_j, Mf_\ell\rangle)$. Using the same argument by developing $\langle f_j + if_\ell, M_m(f_j+if_\ell)\rangle$ shows that $\lim_{m\to \infty}\Im(\langle f_j, M_m f_\ell\rangle) = \Im(\langle f_j, Mf_\ell\rangle)$. Overall,
    \begin{equation}
        \forall 1\leqslant j,\ell \leqslant k,\quad \lim_{m\to \infty}\langle f_j, M_m f_\ell\rangle = \langle f_j, Mf_\ell\rangle.
    \label{eq:pointwise_cv}
    \end{equation}
     Now, observe that \[(g\in \mathrm{Span}(f_1, \hdots, f_k)  \hbox{ and } \|g\|_2=1) \Leftrightarrow (\exists (a_1, \hdots, a_k)\in \mathbb C^k, \;g = \sum_{j=1}^k a_j f_j \hbox{ and } \sum_{j=1}^k |a_j|^2 = 1).\] Thus,
     \begin{align*}
         \underset{\|g\|_2=1}{\max_{g\in \mathrm{Span}(f_1, \hdots, f_k) }} |\langle g, M_m g\rangle -\langle g, M g\rangle|&\leqslant \max_{\|a\|_2=1} \sum_{i,j=1}^k |a_i a_j| |\langle f_i, (M_m-M) f_j\rangle|\\
         & \leqslant k \max_{1\leqslant i,j\leqslant k}|\langle f_i, (M_m-M) f_j\rangle|\\
        &\xrightarrow{m\to \infty}0\quad \hbox{according to \eqref{eq:pointwise_cv}}.
     \end{align*}
     So,
     \begin{align}
         \lim_{m\to\infty}\underset{\|g\|_2=1}{\max_{g\in \mathrm{Span}(\Pi_m(f_1), \hdots, \Pi_m(f_k)) }} \langle g, M_m g\rangle &= \lim_{m\to\infty}\underset{\|g\|_2=1}{\max_{g\in \mathrm{Span}(f_1, \hdots, f_k) }} \langle g, M_m g\rangle\nonumber\\
         &= \underset{\|g\|_2=1}{\max_{g\in \mathrm{Span}(f_1, \hdots, f_k) }} \langle g, M g\rangle\nonumber\\
         &= \sigma_k^\uparrow(M)\label{eq:bound_vp}.
     \end{align}
     Note that $\mathrm{Span}(\Pi_m(f_1), \hdots, \Pi_m(f_k)) \subseteq H_m$. Moreover, for $m$ large enough, we have that $\dim \mathrm{Span}(\Pi_m(f_1), \hdots, \Pi_m(f_k)) = k$.
     Therefore, according to Theorem \ref{thm:minmax}, \[\sigma_k^\uparrow(M_m) = \underset{\dim H = k}{\min_{H\subseteq H_m}} \;  \underset{\|g\|_2=1}{\max_{g\in H }} \langle g, M_m g\rangle \leqslant \underset{\|g\|_2=1}{\max_{g\in \mathrm{Span}(\Pi_m(f_1), \hdots, \Pi_m(f_k)) }} \langle g, M_m g\rangle .\] Combining this inequality with identity \eqref{eq:bound_vp} shows that 
     $\limsup_{m\to\infty}\sigma_k^\uparrow(M_m) \leqslant \sigma_k^\uparrow(M)$.
     Equivalently, \[\liminf_{m\to\infty}\sigma_k^\downarrow(M_m^{-1}) \geqslant  \sigma_k^\downarrow(M^{-1}).\] 
     Finally, by Lemma \ref{lem:bounding_spectrum}, we have $\sigma_k^\downarrow(M_m^{-1}) \leqslant \sigma_k^\downarrow(M^{-1})$. We conclude that $\lim_{m\to\infty}\sigma_k^\downarrow(M_m^{-1}) = \sigma_k^\downarrow(M^{-1}).$
\end{proof}

\begin{lem}[Eigenfunctions convergence]
    Let $(f_{j,m})_{j\in\mathbb N^\star}$ be the eigenvectors of $M_m$ associated with the eigenvalues $(\sigma^\uparrow_j(M_m))_{j\in\mathbb N^\star}$. Let $E_j = \ker(M-\sigma^\uparrow_j(M)\mathrm{Id})$.
    Then
    \[\forall j\in \mathbb N^\star, \quad \lim_{m\to\infty} \min_{y\in E_j}\|f_{j,m}-y\|_2 = 0.\]
    \label{lem:eigenfunction_cv}
\end{lem}
\begin{proof}
Let $(f_j)_{j\in\mathbb N^\star}$ be the eigenvectors of $M$ associated with the eigenvalues $(\sigma^\uparrow_j(M))_{j\in\mathbb N^\star}$.

    The proof proceeds by contradiction. Assume that the lemma is false, and 
    consider the minimum integer $p\in \mathbb N^\star$ such that $\limsup_{m\to\infty} \min_{y\in E_p}\|f_{p,m}-y\|_2  > 0$. 
    Let $k_1 < p$ be the largest integer such that $\sigma^\uparrow_{k_1}(M_m) < \sigma^\uparrow_p(M_m)$, and let $k_2 > p$ be the smallest integer such that $\sigma^\uparrow_{k_2}(M_m) > \sigma^\uparrow_p(M_m)$. Observe that $E_p = \mathrm{Span}(f_{k_1+1}, \hdots, f_{k_2-1})$.

    Let $k_1 < j < k_2$. We know that
    $\langle f_j, M_m f_j\rangle = \sum_{\ell \in \mathbb N^\star} \sigma^\uparrow_j(M_m) |\langle f_{\ell, m}, f_j\rangle|^2$ from diagonalizing $M_m$. 
    The minimality assumption on $j$ ensures that, for all $\ell \leqslant k_1$, $\lim_{m\to\infty} \min_{y\in E_\ell}\|f_{\ell,m}-y\|_2 = 0$. Since $E_j \subseteq \mathrm{Span}(f_1, \hdots, f_{k_1})$, we deduce that $\forall \ell \leqslant k_1$, $\lim_{m\to\infty} \langle f_{\ell, m}, f_j\rangle = 0$. Thus, 
    \begin{equation}
        \sum_{\ell\geqslant  k_1}  |\langle f_{\ell, m}, f_j\rangle|^2 = 1+ o_{m\to \infty}(1),\label{eq:technical1}
    \end{equation}
    and
    \begin{equation}
        \sum_{k_1\leqslant \ell < k_2} \sigma^\uparrow_j(M_m) |\langle f_{\ell, m}, f_j\rangle|^2 = o_{m\to \infty}(1).\label{eq:technical2}
    \end{equation}
    By Lemma \ref{lem:convergence_spectrum}, using $|\langle f_{\ell, m}, f_j\rangle|^2 \leqslant 1$, we have 
    \begin{equation}
        \sum_{k_1\leqslant \ell \leqslant k_2} \sigma^\uparrow_j(M_m) |\langle f_{\ell, m}, f_j\rangle|^2 = \sigma^\uparrow_j(M) \sum_{k_1\leqslant \ell < k_2}  |\langle f_{\ell, m}, f_j\rangle|^2 + o_{m\to \infty}(1).\label{eq:technical3}
    \end{equation}
    Combining \eqref{eq:technical2} and \eqref{eq:technical3}, we deduce that 
    \[\langle f_j, M_m f_j\rangle = o_{m\to \infty}(1) +\sigma^\uparrow_j(M)\sum_{k_1\leqslant \ell < k_2} |\langle f_{\ell, m}, f_j\rangle|^2 + \sum_{ \ell \geqslant  k_2} \sigma^\uparrow_j(M_m) |\langle f_{\ell, m}, f_j\rangle|^2.\]
    Moreover, identity
    \eqref{eq:pointwise_cv} ensures that $\langle f_j, M_m f_j\rangle =  \sigma^\uparrow_j(M) +  o_{m\to \infty}(1).$
    Thus, 
    \begin{equation*}
        \sigma^\uparrow_j(M) (1 - \sum_{k_1\leqslant \ell < k_2} |\langle f_{\ell, m}, f_j\rangle|^2) = o_{m\to \infty}(1) +  \sum_{ \ell \geqslant  k_2} \sigma^\uparrow_j(M_m) |\langle f_{\ell, m}, f_j\rangle|^2.
    \end{equation*}
    However, according to Lemma \ref{lem:convergence_spectrum}, there is $\varepsilon > 0$ such that, for $m$ large enough,
    \[\forall \ell \geqslant  k_2, \quad \sigma^\uparrow_j(M_m) \geqslant  \sigma^\uparrow_k(M) + \varepsilon.\]
    Hence, 
    \begin{equation*}
        \sigma^\uparrow_j(M) (1 - \sum_{k_1\leqslant \ell < k_2} |\langle f_{\ell, m}, f_j\rangle|^2) \geqslant  o_{m\to \infty}(1) +  (\sigma^\uparrow_j(M)+\varepsilon)\sum_{ \ell \geqslant  k_2}  |\langle f_{\ell, m}, f_j\rangle|^2.
    \end{equation*}
    Combining this inequality with \eqref{eq:technical1}, this means that
    \begin{equation*}
        0 \geqslant  o_{m\to \infty}(1) +  \varepsilon\sum_{ \ell \geqslant  k_2}  |\langle f_{\ell, m}, f_j\rangle|^2.
    \end{equation*}
    Thus, $\lim_{m\to\infty} \sum_{ \ell \geqslant  k_2}  |\langle f_{\ell, m}, f_j\rangle|^2 = 0$ and $\lim_{m\to\infty} \sum_{ k_1 \leqslant \ell < k_2}  |\langle f_{\ell, m}, f_j\rangle|^2 = 1$. 

    We deduce that, for all $k_1 < j < k_2$, $\lim_{m\to\infty} \min_{y\in \mathrm{Span}(f_{k_1+1, m}, \hdots, f_{k_2-1, m})}\|f_{j}-y\|_2 = 0$. By symmetry of the $\ell^2$-distance between two spaces of the same dimension $k_2-k_1-1$, for all $k_1 < j < k_2$, $\lim_{m\to\infty} \min_{y\in \mathrm{Span}(f_{k_1+1}, \hdots, f_{k_2-1})}\|f_{j,m}-y\|_2 = 0$. This contradicts the fact that $\limsup_{m\to\infty} \min_{y\in E_p}\|f_{p,m}-y\|_2  > 0$.
\end{proof}

\begin{lem}[Convergence of $M_m^{-1}$] One has
    \[\lim_{m \to \infty} |||M^{-1}-M_m^{-1}|||_2 = 0.\]
    \label{lem:operator_norm}
\end{lem}
\begin{proof}
    Let $(f_{j,m})_{j\in\mathbb N^\star}$ be the eigenvectors of $M_m$, each associated with the corresponding eigenvalues $(\sigma^\uparrow_j(M_m))_{j\in\mathbb N^\star}$. Let $(f_j)_{j\in\mathbb N^\star}$ be the eigenvectors of $M$, each associated with the eigenvalues $(\sigma^\uparrow_j(M))_{j\in\mathbb N^\star}$.
    By Lemma \ref{lem:bounding_spectrum}, $\sigma^\downarrow_j(M_m^{-1}) \leqslant \sigma^\downarrow_j(M^{-1})$; by Lemma \ref{lem:convergence_spectrum}, $\lim_{m\to\infty} \sigma^\downarrow_j(M_m^{-1}) = \sigma^\downarrow_j(M^{-1})$; and by Lemma \ref{lem:eigenfunction_cv} $\lim_{m\to\infty} \min_{y\in E_j}\|f_{j,m}-y\|_2 = 0.$

    Notice that $M_m^{-1} = \sum_{\ell \in \mathbb N } \sigma^\downarrow_j(M_m^{-1}) \langle f_{j,m}, \cdot\rangle f_{j,m}$ and that $M^{-1} = \sum_{\ell \in \mathbb N } \sigma^\downarrow_j(M^{-1}) \langle f_{j}, \cdot\rangle f_{j}$. Let $g\in L^2([-2L,2L]^d, \mathbb C)$ be such that $\|g\|_2 = 1$.
    Then,
    \begin{align*}
        \|(M^{-1}-M^{-1}_m)g\|_2 &\leqslant \sum_{j \in \mathbb N^\star} \|\sigma^\downarrow_j(M_m^{-1}) \langle f_{j,m}, g\rangle f_{j,m} - \sigma^\downarrow_j(M^{-1}) \langle f_{j}, g\rangle f_{j}\|_2\\
        &\leqslant \sum_{j \in \mathbb N^\star} (\sigma^\downarrow_j(M_m^{-1})-\sigma^\downarrow_j(M^{-1}))\| \langle f_{j,m}, g\rangle f_{j,m}\|_2 \\
        &\quad + \sum_{j \in \mathbb N^\star}\sigma^\downarrow_j(M^{-1})\| \langle f_{j,m}- f_{j}, g\rangle f_{j,m}\|_2\\
        &\quad + \sum_{j \in \mathbb N^\star}\sigma^\downarrow_j(M^{-1})\| \langle f_{j}, g\rangle (f_{j,m}-f_{j})\|_2.
    \end{align*}
    Since $\|f_{j,m}\|_2 = \|f_j\|_2 = 1$, it follows that $| \langle f_{j,m}, g\rangle| \leqslant 1$. Additionally, by the Cauchy-Schwarz inequality, $| \langle f_{j}, g\rangle| \leqslant 1$ and $| \langle f_{j,m}- f_{j}, g\rangle | \leqslant \| (f_{j,m}-f_{j})\|_2$. Thus, the above inequality can be simplified as 
    \begin{align*}
        \|(M^{-1}-M^{-1}_m)g\|_2 
        &\leqslant \sum_{j \in \mathbb N^\star} (\sigma^\downarrow_j(M_m^{-1})-\sigma^\downarrow_j(M^{-1}) + 2\sigma^\downarrow_j(M^{-1}) \| f_{j,m}-f_{j}\|_2).
    \end{align*}
    Thus, 
    \begin{align*}
        |||M^{-1}-M^{-1}_m|||_2 
        &\leqslant \sum_{j \in \mathbb N^\star} (\sigma^\downarrow_j(M_m^{-1})-\sigma^\downarrow_j(M^{-1}) + 2\sigma^\downarrow_j(M^{-1}) \| f_{j,m}-f_{j}\|_2).
    \end{align*}
    Clearly, since $|\sigma^\downarrow_j(M_m^{-1})-\sigma^\downarrow_j(M^{-1})| \leqslant 2\sigma^\downarrow_j(M^{-1})$ and $\| f_{j,m}-f_{j}\|_2 \leqslant 2$,  \[|\sigma^\downarrow_j(M_m^{-1})-\sigma^\downarrow_j(M^{-1}) + 2\sigma^\downarrow_j(M^{-1}) \| f_{j,m}-f_{j}\|_2|\leqslant 4\sigma^\downarrow_j(M^{-1}).\] Moreover, $\sum_{j \in \mathbb N^\star} \sigma^\downarrow_j(M^{-1}) < \infty$ \citep[] [Proposition B.6]{doumeche2024physicsinformed}. 
    Thus, since we have that $\lim_{m\to \infty}|\sigma^\downarrow_j(M_m^{-1})-\sigma^\downarrow_j(M^{-1})| = \lim_{m\to\infty }\| (f_{j,m}-f_{j})\|_2 = 0$, we conclude with the dominated convergence theorem that $\lim_{m\to\infty} |||M^{-1}-M^{-1}_m|||_2 = 0$, as desired.
\end{proof}

\subsection{Operator norms of $C_m$ and $C$}
\begin{lem}
    One has $|||C|||_2 \leqslant 1$ and $|||C_m|||_2\leqslant 1$, for all $m\in \mathbb N^\star$.
    \label{lem:C_op_norm}
\end{lem}
\begin{proof}
    Let $g\in L^2([-2L,2L]^d, \mathbb C)$. Then, by definition, 
    $Cg = 1_\Omega g$, and $\|Cg\|_2 =  \|1_\Omega  g\|_2 \leqslant \|g\|_2$. Therefore, 
    $|||C|||_2 \leqslant 1$.
    
    Let $m\in \mathbb N^\star$. Then, since $C_m: L^2([-2L,2L]^d, \mathbb C) \to L^2([-2L,2L]^d, \mathbb C)$ is a positive Hermitian compact operator, Theorem \ref{thm:minmax} states that 
    \begin{align*}
        \sigma_1^\downarrow( C_m) &= \underset{\|h\|_2=1}{\max_{h\in L^2([-2L,2L]^d, \mathbb C)}}  \langle h,  C_m h\rangle =\underset{\|h\|_2=1}{\max_{h\in L^2([-2L,2L]^d, \mathbb C)}} \|\Pi_m h\|_2^2 \leqslant \underset{\|h\|_2=1}{\max_{h\in L^2([-2L,2L]^d, \mathbb C)}} \|h\|_2^2=1.
    \end{align*}
    Since $\sigma_1^\downarrow( C_m^2) = \sigma_1^\downarrow( C_m)^2 \leqslant 1$, we deduce that
    \begin{align*}
        1 &\geqslant  \sigma_1^\downarrow( C_m)^2 = \underset{\|h\|_2=1}{\max_{h\in L^2([-2L,2L]^d, \mathbb C)}}  \langle h,  C_m^2 h\rangle =\underset{\|h\|_2=1}{\max_{h\in L^2([-2L,2L]^d, \mathbb C)}}  \|C_m h\|_2^2.
    \end{align*}
    This shows that $|||C_m|||_2\leqslant 1$.
\end{proof}

\subsection{Proof of Theorem \ref{thm:convergence_eff_dim}}

Note that if $H$ is a linear subspace of $H_m$, then $C_mH$ is also a subspace of $H_m$, and $\dim H \geqslant  \dim C_mH$. Therefore, 
    \begin{align}
    \sigma_k^\downarrow(C_mM_m^{-1}C_m) &= \underset{\dim H = k}{\max_{H\subseteq H_m}} \;\underset{\|g\|_2=1}{\min_{g\in H }} \langle g, C_mM_m^{-1}C_m g\rangle\nonumber\\
    &= \underset{\dim H = k}{\max_{H\subseteq H_m}}\; \underset{\|g\|_2=1}{\min_{g\in H }} \langle (C_mg), M_m^{-1} (C_mg)\rangle\nonumber\\
    &\leqslant \underset{\dim H = k}{\max_{H\subseteq H_m}}\; \underset{\|g\|_2=1}{\min_{g\in H }} \langle g, M_m^{-1} g\rangle\nonumber\\
    &= \sigma_k^\downarrow(M_m^{-1})\nonumber\\
    &\leqslant \sigma_k^\downarrow(M^{-1}).\label{eq:cvd_bound}
    \end{align} 
Moreover, according to Lemma \ref{lem:operator_norm}, one has $|||M_m^{-1}-M^{-1}|||_2 \to 0$. Thus,  $\sup_{\|g\|_2=1}||(M_m^{-1}-M^{-1})(g)||_2 \to 0$. Using Lemma \ref{lem:C_op_norm}, we see that 
\begin{align*}
    &\|C_mM_m^{-1}C_mg - CM^{-1}Cg\|_2 \\
    &\leqslant  \|(C_m-C)M^{-1}Cg\|_2 +\|C_m(M_m^{-1}C_m - M^{-1}C)g \|_2 \\
    &\leqslant  |||(C_m-C)M^{-1}|||_2 +\|(M_m^{-1}C_m - M^{-1}C)g \|_2\\
    &\leqslant  |||(C_m-C)M^{-1}|||_2 +\|(M_m^{-1}-M^{-1})C_m g \|_2 + \|M^{-1}(C_m -C)g \|_2\\
    &\leqslant  |||(C_m-C)M^{-1}|||_2 +|||M_m^{-1}-M^{-1} |||_2 + |||M^{-1}(C_m -C)|||_2.
\end{align*}
Thus, \[|||C_mM_m^{-1}C_m - CM^{-1}C|||_2\leqslant |||(C_m-C)M^{-1}|||_2 +|||M_m^{-1}-M^{-1} |||_2 + |||M^{-1}(C_m -C)|||_2.\]
By diagonalizing $M^{-1}$ and using the facts that $\sum_{\ell\in\mathbb N^\star}\sigma^\downarrow_\ell(M^{-1}) < \infty$ and that $\lim_{m\to\infty}\|(C_m-C)f\|_2 = 0$ $\forall f \in L^2([-2L,2L]^d)$, it is easy to see that \[\lim_{m\to\infty} |||(C_m-C)M^{-1}|||_2 = \lim_{m\to\infty} |||M^{-1}(C_m-C)|||_2 = 0.\] Applying Lemma \ref{lem:operator_norm}, we deduce that 
\[\lim_{m\to\infty}|||C_mM_m^{-1}C_m - CM^{-1}C|||_2 = 0.\]
But, by Theorem \ref{thm:minmax},
\[\sigma_k^\downarrow(C_mM_m^{-1}C_m) = \underset{\dim H = k}{\max_{H\subseteq L^2([-2L,2L]^d, \mathbb C)}} \;\underset{\|g\|_2=1}{\min_{g\in H }} \langle g, C_mM_m^{-1}C_m g\rangle,\]
and 
\[\sigma_k^\downarrow(CM^{-1}C) = \underset{\dim H = k}{\max_{H\subseteq L^2([-2L,2L]^d, \mathbb C)}} \;\underset{\|g\|_2=1}{\min_{g\in H }} \langle g, CM^{-1}C g\rangle.\]
Clearly, for all $g\in L^2([-2L,2L]^d, \mathbb C)$,
\begin{align*}
|\langle g, CM^{-1}C g\rangle - \langle g, C_mM_m^{-1}C_m g\rangle| &= |\langle g, (CM^{-1}C-C_mM_m^{-1}C_m) g\rangle|\\&
\leqslant |||C_mM_m^{-1}C_m - CM^{-1}C|||_2.
\end{align*}
Therefore,
\[|\sigma_k^\downarrow(C_mM_m^{-1}C_m) - \sigma_k^\downarrow(CM^{-1}C)|\leqslant |||C_mM_m^{-1}C_m - CM^{-1}C|||_2,\]
and, in turn,
$\lim_{m\to\infty}\sigma_k^\downarrow(C_mM_m^{-1}C_m) = \sigma_k^\downarrow(CM^{-1}C)$.

To conclude the proof, observe that, on the one hand, 
\[
\frac{1}{1+\sigma_k^\downarrow(C_mM_m^{-1}C_m)^{-1}} = \frac{\sigma_k^\downarrow(C_mM_m^{-1}C_m)}{1+\sigma_k^\downarrow(C_mM_m^{-1}C_m)} \leqslant \sigma_k^\downarrow(C_mM_m^{-1}C_m) \leqslant \sigma_k^\downarrow(M^{-1}),
\] with $\sum_{k\in\mathbb N^\star} \sigma_k^\downarrow(M^{-1}) < \infty$.
On the other hand, 
\[\lim_{m\to\infty}\frac{1}{1+\sigma_k^\downarrow(C_mM_m^{-1}C_m)^{-1}} = \frac{1}{1+\sigma_k^\downarrow(CM^{-1}C)^{-1}},
\]
by continuity on $\mathbb R^+$ of the function $x\mapsto \frac{x}{1+x}$. Thus, applying the dominated convergence theorem, we are led to
\[\lim_{m\to\infty}\sum_{\lambda\in \sigma(C_mM_m^{-1}C_m)} \frac{1}{1+\lambda^{-1}} = \sum_{\lambda\in \sigma(CM^{-1}C)} \frac{1}{1+\lambda^{-1}}.\]

\section{Experiments}
\subsection{Numerical precision}
Enabling high numerical precision is crucial for efficient kernel inversion. Setting the default precision to \textit{Float32} and \textit{Complex64} can lead to significant numerical errors when approximating the kernel. For example, consider the harmonic oscillator case with $d=1$, $s=2$, and the operator $\mathscr D = \frac{d^2}{dx^2}u+\frac{d}{dx}u+u$, with $\lambda_n=0.01$ and $\mu_n=1$. Figure \ref{fig:float_precision} (left) shows the spectrum of $C_mM_m^{-1}C_m$ using \textit{Float32} precision, while Figure \ref{fig:float_precision} (right) shows the same spectrum with \textit{Float64} precision.
\begin{figure}
    \centering
    \includegraphics[scale=0.45]{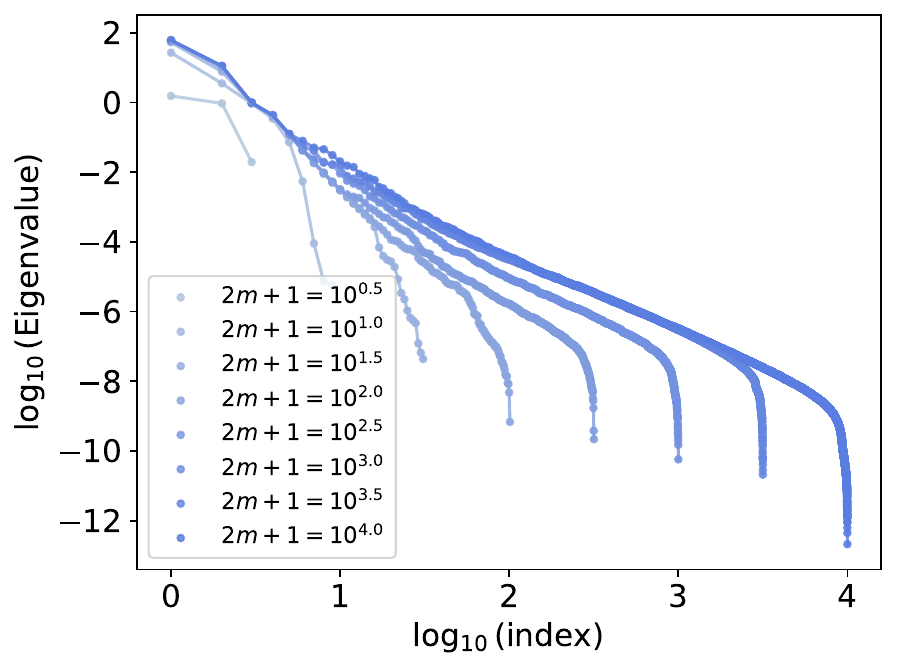}
    \includegraphics[scale=0.45]{oscillator_spectrum.pdf}
    \caption{Spectrum of $C_mM_m^{-1}C_m$. \textbf{Left:} \textit{Float32} precision. \textbf{Right:} \textit{Float64} precision.}
    \label{fig:float_precision}
\end{figure}
It is evident that with \textit{Float32}, the diagonalization results in lower eigenvalues compared to \textit{Float64}. 
In more physical terms, some energy of the matrix is lost when the last digits of the matrix coefficients are ignored. This leads to a problematic interpretation of the situation, as the \textit{Float64} estimation of the eigenvalues shows a clear convergence of the spectrum of $C_m M_m^{-1} C_m$, whereas the \textit{Float32} estimation appears to indicate divergence.

\subsection{Convergence of the effective dimension approximation}
\label{sec:eff_dim_m}
The PIKL algorithm relies on a Fourier approximation of the PIML kernel, as developed in Section~\ref{sec:finite_approx}. 
The precision of this approximation is determined by the number $m$ of Fourier modes used to compute the kernel.
However, determining an \textit{optimal} value of $m$ for a specific regression problem is challenging. There is a trade-off between the accuracy of the kernel estimation, which improves with higher values of $m$, and the computational complexity of the algorithm, which also increases with m. There is a trade-off between the accuracy of the kernel approximation, which improves with higher values of $m$, and the computational complexity of the algorithm, which also increases with $m$. 

An interesting tool to leverage here is the effective dimension, as it captures the underlying degrees of freedom of the PIML problem and, consequently, the precision of the method.  
Theorem~\ref{thm:convergence_eff_dim} states that the estimation of the effective dimension on $H_m$ converges to the effective dimension on $H^s(\Omega)$ as $m$ increases to infinity. Therefore, the smallest value $m^\star$ at which the effective dimension stabilizes is a strong candidate for balancing accuracy and computational complexity.

Figures~\ref{fig:eff_dim_cv_1d}, \ref{fig:eff_dim_cv_osc}, and \ref{fig:eff_dim_cv_heat_sm} illustrate the convergence of the effective dimension estimation, using the eigenvalues of $C_m M_M^{-1}C_m$, as $m$ increases, for different values of $n$.
\begin{figure}
    \centering
    \includegraphics[width=0.5\linewidth]{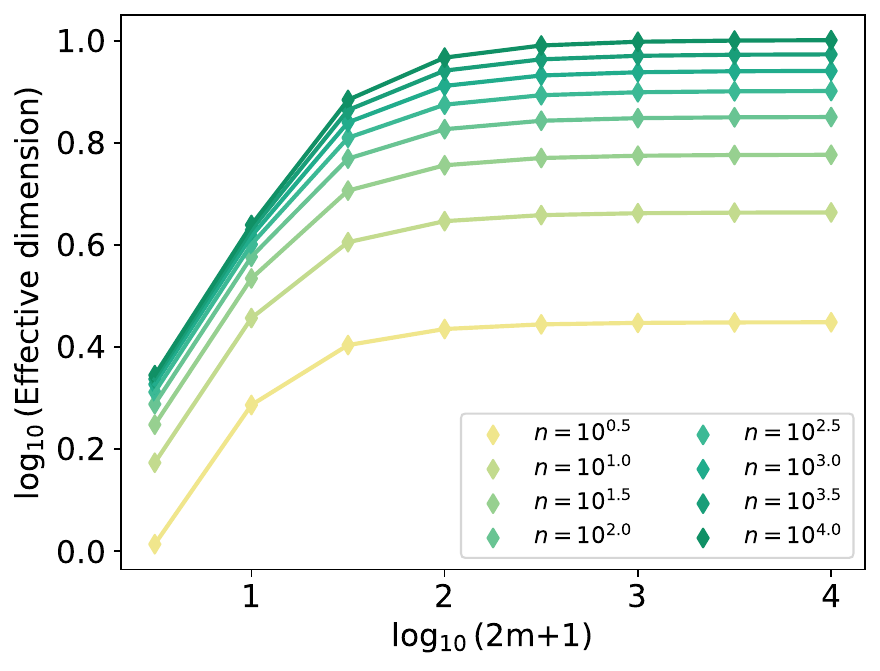}
    \caption{Convergence of the effective dimension as $m$ grows for $\mathscr D = \frac{d}{dx}$.}
    \label{fig:eff_dim_cv_1d}
\end{figure}
\begin{figure}
    \centering
    \includegraphics[width=0.5\linewidth]{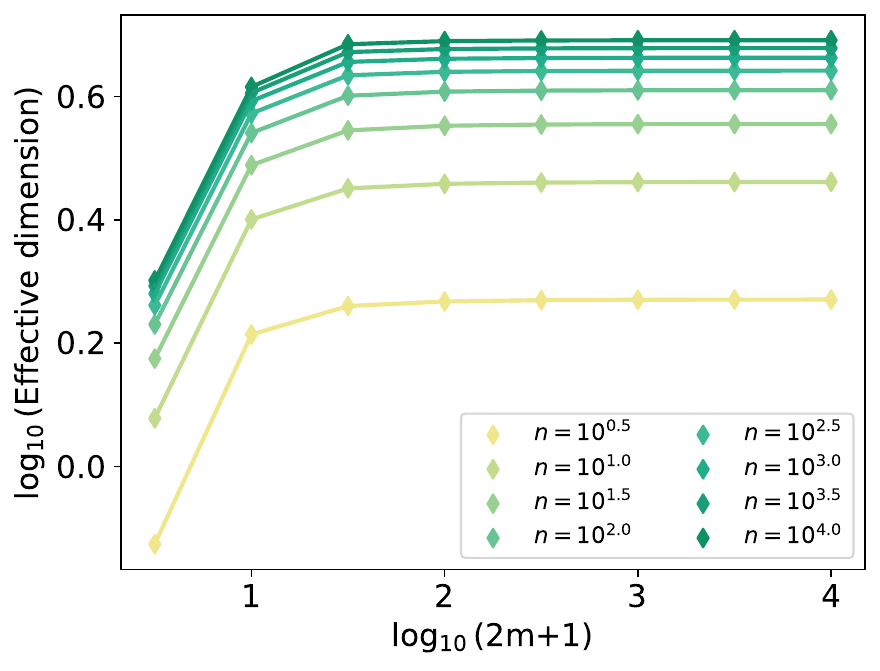}
    \caption{Convergence of the effective dimension as $m$ grows for the harmonic oscillator}
    \label{fig:eff_dim_cv_osc}
\end{figure}
\begin{figure}
    \centering
    \includegraphics[width=0.5\linewidth]{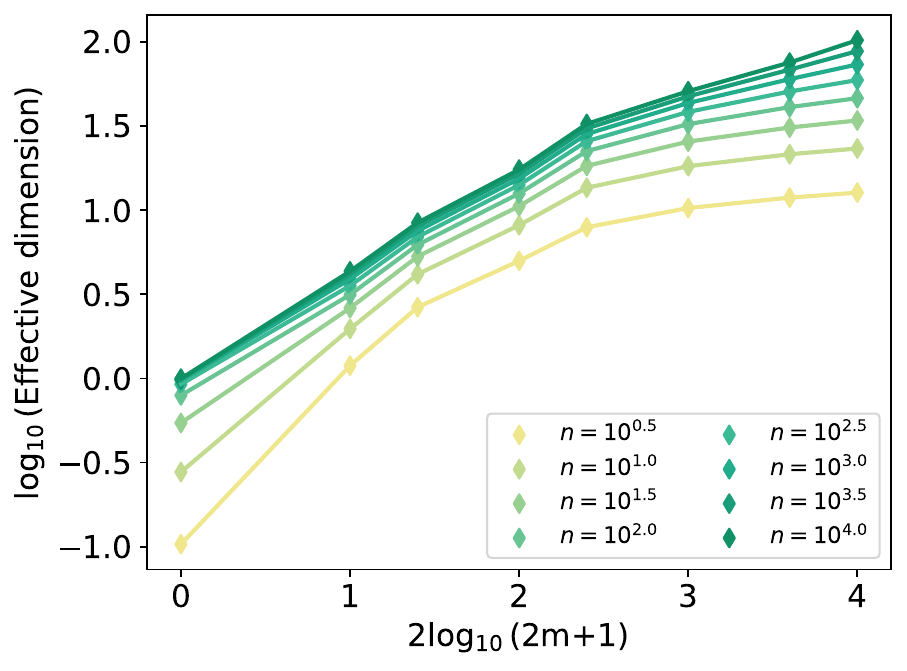}
    \caption{Convergence of the effective dimension as $m$ grows for the heat equation on the disk.}
    \label{fig:eff_dim_cv_heat_sm}
\end{figure}
These figures provide insights into the PIK algorithm. As expected, the Fourier approximations converge more slowly as the dimension $d$ increases. Specifically, Figures ~\ref{fig:eff_dim_cv_1d} and \ref{fig:eff_dim_cv_osc} show that in dimension $d=1$, $m^\star \simeq 10^2$ for $n \leqslant 10^4$, while Figure~\ref{fig:eff_dim_cv_heat_sm} indicates that in dimension $d=2$, $m^\star \simeq 10^2$ for $n \leqslant 10^3$.

\subsection{Numerical schemes}
\label{sec:numerical}
We detail below the numerical schemes used as benchmarks in Section~\ref{sec:PDE_solving} for solving the wave equation. All these numerical schemes are constructed by discretizing the domain $\Omega = [0,1]^2$ into the grid $(\ell_1^{-1} \mathbb Z/\ell_1 \mathbb Z) \times (\ell_2^{-1} \mathbb Z/\ell_2 \mathbb Z)$. The initial and boundary conditions are then enforced on $n = 2\ell_1 + \ell_2$ points, with the approximation $\hat{f}_{n}$ defined accordingly as
\begin{itemize}
    \item for all $0\leqslant \ell \leqslant \ell_2$, $\hat f_{n}(0, \ell / \ell_2) = \sin(\pi \ell / \ell_2) + \sin(4 \pi \ell / \ell_2)/2$,
    \item for all $0\leqslant \ell \leqslant \ell_1$, $\hat f_{n}(\ell / \ell_1, 0) = 0$,
    \item for all $0\leqslant \ell \leqslant \ell_1$, $\hat f_{n}(\ell / \ell_1, 1) = 0$.
\end{itemize}
Let the discrete Laplacian $\Delta_{(\ell_1,\ell_2)}$ be defined for all $(a,b) \in \mathbb Z/\ell_1 \mathbb Z \times \mathbb Z/\ell_2 \mathbb Z$ by 
\[(\Delta_{(\ell_1,\ell_2)} \hat f_{n})(a/\ell_1, b / \ell_2) = \ell_2^{2}(\hat f_{n}(a/\ell_1, (b+1) / \ell_2) - 2\hat f_{n}(a/\ell_1, b / \ell_2) + \hat f_{n}(a/\ell_1, (b-1) / \ell_2)).\]
If $f^\star \in C^2([0,1]^2)$, its  Taylor expansion leads to $(\Delta_{(\ell_1,\ell_2)} f^\star)(a/\ell_1, b / \ell_2) = \partial^2_{x,x} f^\star(a/\ell_1, b / \ell_2) + o_{\ell_2 \to 0}(1)$.
Similarly, let the second-order time partial derivative operative $\partial^2_{t,t, (\ell_1,\ell_2)}$ be defined for all $(a,b) \in \mathbb Z/\ell_1 \mathbb Z \times \mathbb Z/\ell_2 \mathbb Z$ by \[(\partial^2_{t,t, (\ell_1,\ell_2)} \hat f_{n})(a/\ell_1, b / \ell_2) = \ell_2^{2}(\hat f_{n}((a+1)/\ell_1, b / \ell_2) - 2\hat f_{n}(a/\ell_1, b / \ell_2) + \hat f_{n}((a-1)/\ell_1, b / \ell_2)).\]
\paragraph{Euler explicit.} The Euler explicit scheme is initialized using the Taylor expansion $f(t, x) = f(0, x) + t \partial_t f(0, x) + t^2 \partial^2_{t,t} f(0, x)/2 + o_{t\to 0}(t^2)$. With the initial condition $\partial_t f(0, x) = 0$ and the wave equation $\partial^2_{t,t} f(0, x) = 4 \partial^2_{x,x}f(0,x)$, this simplifies to $f(t, x) = f(0, x) +  2 t^2 \partial^2_{x,x} f(0, x) + o_{t\to 0}(t^2)$. This leads to the initialization 
\[\forall 0\leqslant b \leqslant \ell_2, \quad \hat f_{n}(1/\ell_1, b / \ell_2) = \hat f_{n}(0, b / \ell_2) + 2\ell_1^{-2}(\Delta_{(\ell_1,\ell_2)} \hat f_{n})(0, b / \ell_2).\]
The wave equation $\partial^2_{t,t} f^\star = 4 \partial^2_{x,x} f^\star$ can then be discretized as $\partial^2_{t,t, (\ell_1,\ell_2)} \hat f_n = 4 \Delta_{(\ell_1,\ell_2)} \hat f_n$. This leads to the explicit Euler recursive formula
\[\hat f_{n}((a+1)/\ell_1, b / \ell_2) = 2\hat f_{n}(a/\ell_1, b / \ell_2) - \hat f_{n}((a-1)/\ell_1, b / \ell_2) + 4\ell_1^{-2} (\Delta_{(\ell_1,\ell_2)} \hat f_{n})(a/\ell_1, b / \ell_2).\]
This formula allows to compute $\hat f_{n}((a+1)/\ell_1, \cdot)$ given the values of $\hat f_{n}(0, \cdot), \hdots, \hat f_{n}(a/\ell_1, \cdot)$.
\paragraph{Runge-Kutta 4.} The RK4 scheme is a numerical scheme applied on both $f^\star$ and its derivative $\partial_t f^\star$. Here, $\hat g_n$ represents the approximation of $\partial_t f^\star$. The initial condition $\partial_t f(0, \cdot)=0$ translates into 
\[\forall 0\leqslant b \leqslant \ell_2, \quad \hat g_{n}(0, b / \ell_2) = 0.\]
To infer $\hat f_{n}((a+1)/\ell_1, \cdot)$ and $\hat g_{n}((a+1)/\ell_1, \cdot)$ given the values of $\hat f_{n}(0, \cdot), \hdots, \hat f_{n}(a/\ell_1, \cdot)$ and $\hat g_{n}(0, \cdot), \hdots, \hat g_{n}(a/\ell_1, \cdot)$, the RK4 scheme introduces intermediate estimates as follows:
\begin{itemize}
    \item $\tilde f_1 = \hat g_{n}(a/\ell_1, \cdot) /\ell_1$,
    \item $\tilde g_1 = 4 (\Delta_{(\ell_1,\ell_2)} \hat f_{n})(a/\ell_1, \cdot)/\ell_1$,
    \item $\tilde f_2 = (\hat g_{n}(a/\ell_1, \cdot) +0.5 \tilde f_1)/\ell_1$,
    \item $\tilde g_2 = 4 (0.5 \tilde g_1 + (\Delta_{(\ell_1,\ell_2)} \hat f_{n})(a/\ell_1, \cdot))/\ell_1$,
    \item $\tilde f_3 = (\hat g_{n}(a/\ell_1, \cdot) +0.5 \tilde f_2)/\ell_1$,
    \item $\tilde g_3 = 4 (0.5 \tilde g_2 + (\Delta_{(\ell_1,\ell_2)} \hat f_{n})(a/\ell_1, \cdot))/\ell_1$,
    \item $\tilde f_4 = (\hat g_{n}(a/\ell_1, \cdot) + \tilde f_3)/\ell_1$,
    \item $\tilde g_4 = 4 (\tilde g_3 + (\Delta_{(\ell_1,\ell_2)} \hat f_{n})(a/\ell_1, \cdot))/\ell_1$,
    \item $\hat f_{n}((a+1)/\ell_1, \cdot) = \hat f_{n}(a/\ell_1, \cdot) + (\tilde f_1+2\tilde f_2+2\tilde f_3+\tilde f_4)/6$,
    \item $\hat g_{n}((a+1)/\ell_1, \cdot) = \hat g_{n}(a/\ell_1, \cdot) + (\tilde g_1+2\tilde g_2+2\tilde g_3+\tilde g_4)/6$.
\end{itemize}
Similarly to the Euler explicit scheme, the RK4 relies on a recursive formulas to compute $\hat f_n$.

\paragraph{Crank-Nicolson}. The CN scheme is an implicit scheme defined as follows. Similar to the Euler explicit scheme, the $\partial_t f(0, \cdot) = 0$ initial condition is implemented as
 \[\forall 0\leqslant \ell \leqslant \ell_2, \quad \hat f_{n}(1/\ell_1, \ell / \ell_2) = \hat f_{n}(0, \ell / \ell_2) + 2\ell_1^{-2}(\Delta_{(\ell_1,\ell_2)} \hat f_{n})(0, \ell / \ell_2).\]
 Then, the recursive formula of this scheme takes the form 
 \[\partial^2_{t,t, (\ell_1,\ell_2)} \hat f_n(a/\ell_1, \cdot) = 2 ((\Delta_{(\ell_1,\ell_2)} \hat f_n)(a/\ell_1, \cdot) + (\Delta_{(\ell_1,\ell_2)} \hat f_n)((a+1)/\ell_1, \cdot)).\]
 This leads to the recursion
 \[\hat f_n((a+1)/\ell_1, \cdot) = (\mathrm{Id} + 2\ell_2^2\ell_1^{-2} \Delta)^{-1}(3 \hat f_n(a/\ell_1, \cdot)- \hat f_n((a-1)/\ell_1, \cdot)) - \hat f_n((a-1)/\ell_1, \cdot),\]
 where $\Delta = \begin{pmatrix}
     2 & -1 & 0 & \hdots & 0\\
     -1 & 2 & -1 & \ddots & \vdots\\
     0 & -1& 2 & \ddots &  0\\
     \vdots & \ddots & \ddots & \ddots & -1 \\
     0 & \hdots & 0 & -1 & 2& 
 \end{pmatrix}$ is the discrete Laplacian matrix.
 
 \subsection{PINN training}
Figures~\ref{fig:nn_noise} and \ref{fig:PINN_BC} illustrates the performance of the PINNs during training while solving the 1d wave equation with noisy boundary conditions.

\begin{figure}
    \centering
    \includegraphics[width=\linewidth]{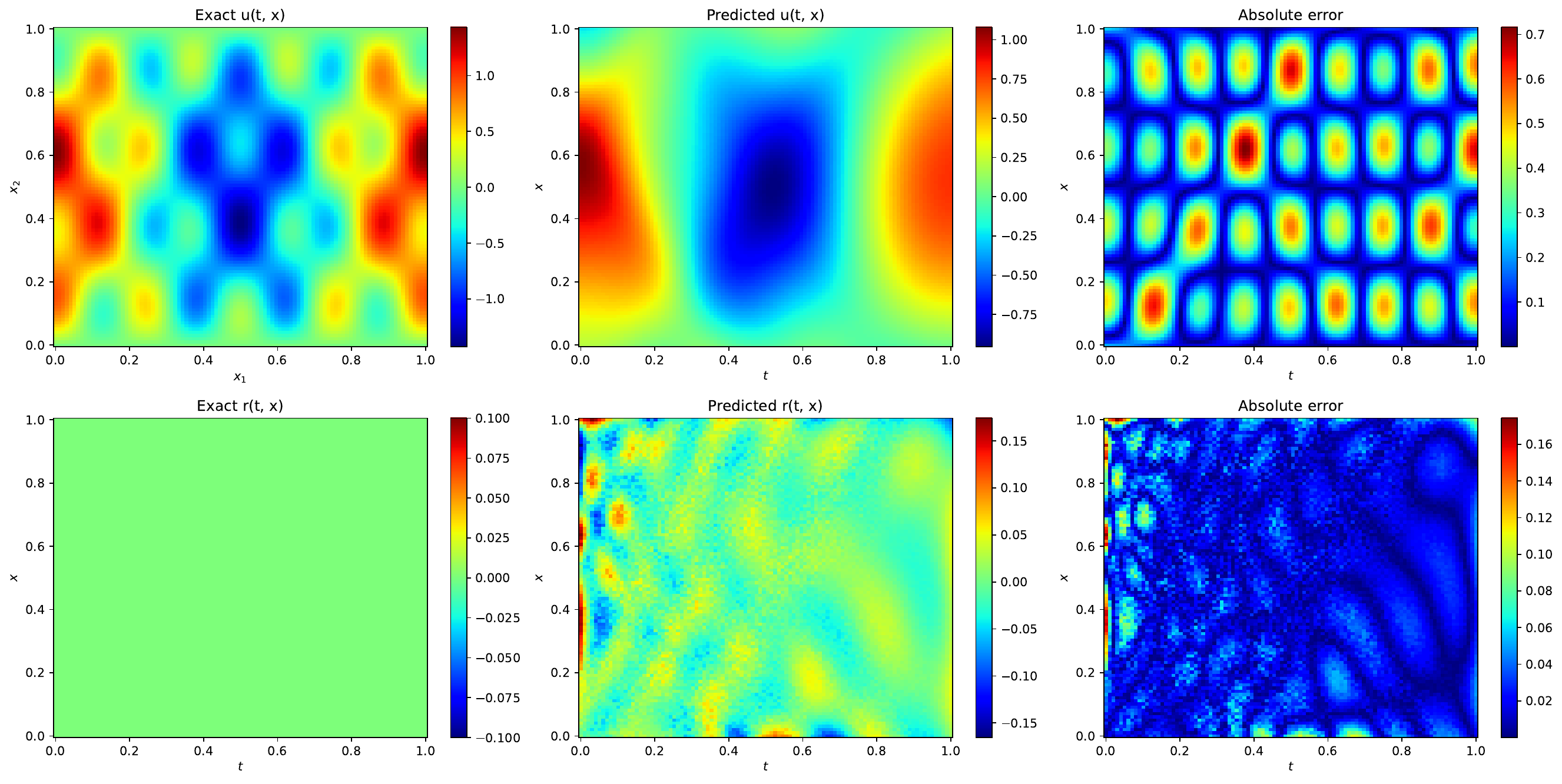}
    \caption{\textbf{Left}: Ground truth. \textbf{Middle}: PINN estimator. \textbf{Right}: Error $=$ PINN - Ground truth.}
    \label{fig:nn_noise}
\end{figure}
\begin{figure}
    \centering
    
    \includegraphics[width=0.7\linewidth]{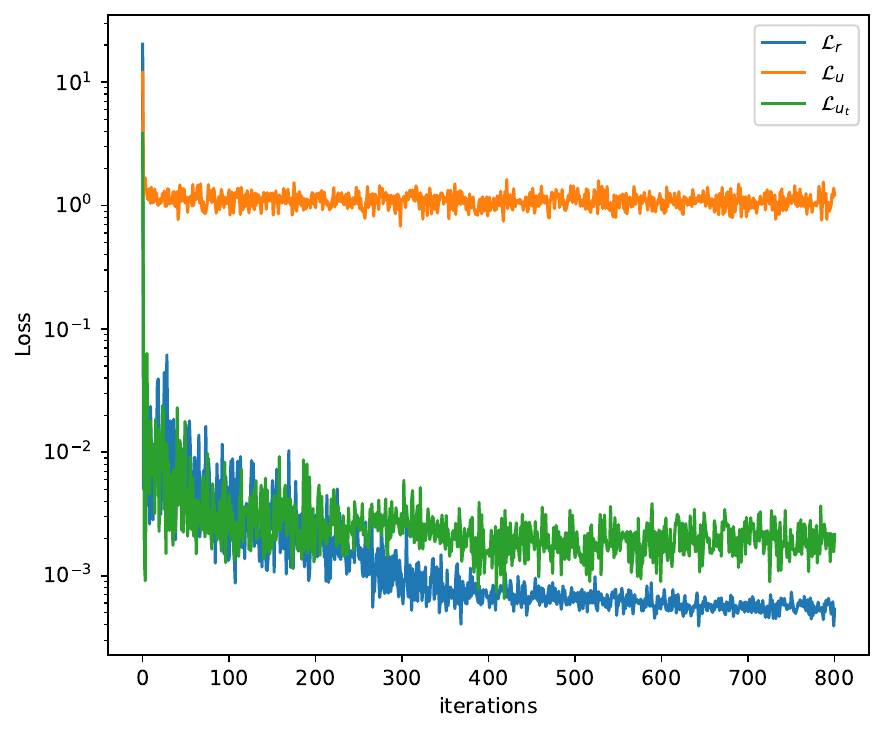}
    \caption{PINN training with noisy boundary conditions.}
    \label{fig:PINN_BC}
\end{figure}

\vskip 0.2in
\bibliography{biblio}

\begin{thebibliography}{38}
\providecommand{\natexlab}[1]{#1}
\providecommand{\url}[1]{\texttt{#1}}
\expandafter\ifx\csname urlstyle\endcsname\relax
  \providecommand{\doi}[1]{doi: #1}\else
  \providecommand{\doi}{doi: \begingroup \urlstyle{rm}\Url}\fi

\bibitem[Agharafeie et~al.(2023)Agharafeie, Ramos, Mendes, and
  Oliveira]{Agharafeie2023from}
R.~Agharafeie, J.~R.~C. Ramos, J.~M. Mendes, and R.~Oliveira.
\newblock From shallow to deep bioprocess hybrid modeling: Advances and future
  perspectives.
\newblock \emph{Fermentation}, 9, 2023.

\bibitem[Arnone et~al.(2022)Arnone, Kneip, Nobile, and
  Sangalli]{arnone2022spatialRegression}
E.~Arnone, A.~Kneip, F.~Nobile, and L.~M. Sangalli.
\newblock Some first results on the consistency of spatial regression with
  partial differential equation regularization.
\newblock \emph{Statistica Sinica}, 32:\penalty0 209--238, 2022.

\bibitem[Arzani et~al.(2021)Arzani, Wang, and D'Souza]{arzani2021uncovering}
A.~Arzani, J.-X. Wang, and R.~M. D'Souza.
\newblock Uncovering near-wall blood flow from sparse data with
  physics-informed neural networks.
\newblock \emph{Physics of Fluids}, 33:\penalty0 071905, 2021.

\bibitem[Badea and Daripa(2001)]{badea2001on}
L.~Badea and P.~Daripa.
\newblock On a boundary control approach to domain embedding methods.
\newblock \emph{SIAM Journal on Control and Optimization}, 40:\penalty0
  421--449, 2001.

\bibitem[Batlle et~al.(2023)Batlle, Chen, Hosseini, Owhadi, and
  Stuart]{batlle2023error}
P.~Batlle, Y.~Chen, B.~Hosseini, H.~Owhadi, and A.~M. Stuart.
\newblock Error analysis of kernel/{GP} methods for nonlinear and parametric
  {PDE}s.
\newblock \emph{arXiv:2305.04962}, 2023.

\bibitem[Bonfanti et~al.(2024)Bonfanti, Bruno, and
  Cipriani]{bonfanti2024challengesnonlinearregimephysicsinformed}
A.~Bonfanti, G.~Bruno, and C.~Cipriani.
\newblock The challenges of the nonlinear regime for physics-informed neural
  networks.
\newblock \emph{arXiv:2402.03864}, 2024.

\bibitem[Bonito et~al.(2024)Bonito, DeVore, Petrova, and
  Siegel]{bonito2024convergenceerrorcontrolconsistent}
A.~Bonito, R.~DeVore, G.~Petrova, and J.~W. Siegel.
\newblock Convergence and error control of consistent {PINN}s for elliptic
  {PDE}s.
\newblock \emph{arXiv:2406.09217}, 2024.

\bibitem[Bracewell(2000)]{bracewell}
R.~Bracewell.
\newblock \emph{The Fourier Transform and its Applications}.
\newblock Electrical Engineering series. McGraw-Hill International Editions,
  Boston, 3 edition, 2000.

\bibitem[Brezis(2010)]{brezis2010functional}
H.~Brezis.
\newblock \emph{Functional Analysis, Sobolev Spaces and Partial Differential
  Equations}.
\newblock Springer, New York, 2010.

\bibitem[Cai et~al.(2021)Cai, Wang, Wang, Perdikaris, and
  Karniadakis]{cai2021physics}
S.~Cai, Z.~Wang, S.~Wang, P.~Perdikaris, and G.~E. Karniadakis.
\newblock Physics-informed neural networks for heat transfer problems.
\newblock \emph{Journal of Heat Transfer}, 143:\penalty0 060801, 2021.

\bibitem[Canuto et~al.(2007)Canuto, Quarteroni, Hussaini, and
  Zang]{canuto2007spectral}
C.~Canuto, A.~Quarteroni, M.~Y. Hussaini, and T.~A. Zang.
\newblock \emph{Spectral Methods}.
\newblock Scientific Computation. Springer, Berlin, 1 edition, 2007.

\bibitem[Caponnetto and Vito(2007)]{caponnetto2007optimal}
A.~Caponnetto and E.~D. Vito.
\newblock Optimal rates for the regularized least-squares algorithm.
\newblock \emph{Foundations of Computational Mathematics}, 7:\penalty0
  331–368, 2007.

\bibitem[Chen et~al.(2021)Chen, Hosseini, Owhadi, and Stuart]{chen2021solving}
Y.~Chen, B.~Hosseini, H.~Owhadi, and A.~M. Stuart.
\newblock Solving and learning nonlinear {PDE}s with {G}aussian processes.
\newblock \emph{Journal of Computational Physics}, 447:\penalty0 110668, 2021.

\bibitem[Cuomo et~al.(2022)Cuomo, Di~Cola, Giampaolo, Rozza, Raissi, and
  Piccialli]{cuomo2022scientific}
S.~Cuomo, V.~S. Di~Cola, F.~Giampaolo, G.~Rozza, M.~Raissi, and F.~Piccialli.
\newblock Scientific machine learning through physics-informed neural networks:
  {W}here we are and what's next.
\newblock \emph{Journal of Scientific Computing}, 92:\penalty0 88, 2022.

\bibitem[Deshpande et~al.(2022)Deshpande, Agarwal, and
  Bhattacharya]{Deshpande2022investigations}
M.~Deshpande, S.~Agarwal, and A.~K. Bhattacharya.
\newblock Investigations on convergence behaviour of physics informed neural
  networks across spectral ranges and derivative orders.
\newblock In \emph{2022 IEEE Symposium Series on Computational Intelligence
  (SSCI)}, pages 1172--1179, 2022.

\bibitem[Doum{\`e}che et~al.(2024{\natexlab{a}})Doum{\`e}che, Bach, Biau, and
  Boyer]{doumeche2024physicsinformed}
N.~Doum{\`e}che, F.~Bach, G.~Biau, and C.~Boyer.
\newblock Physics-informed machine learning as a kernel method.
\newblock In S.~Agrawal and A.~Roth, editors, \emph{Proceedings of Thirty
  Seventh Conference on Learning Theory}, volume 247 of \emph{Proceedings of
  Machine Learning Research}, pages 1399--1450. PMLR, 2024{\natexlab{a}}.

\bibitem[Doum{\`e}che et~al.(2024{\natexlab{b}})Doum{\`e}che, Biau, and
  Boyer]{doumeche2023convergence}
N.~Doum{\`e}che, G.~Biau, and C.~Boyer.
\newblock Convergence and error analysis of {PINN}s.
\newblock \emph{Bernoulli, in press}, 2024{\natexlab{b}}.

\bibitem[Evans(2010)]{evans2010partial}
L.~Evans.
\newblock \emph{Partial Differential Equations}, volume~19 of \emph{Graduate
  Studies in Mathematics}.
\newblock American Mathematical Society, Providence, 2nd edition, 2010.

\bibitem[Karniadakis et~al.(2021)Karniadakis, Kevrekidis, Lu, Wang, Perdikaris,
  and Yang]{karniadakis2021piml}
G.~E. Karniadakis, I.~G. Kevrekidis, L.~Lu, S.~Wang, P.~Perdikaris, and
  L.~Yang.
\newblock Physics-informed machine learning.
\newblock \emph{Nature Reviews Physics}, 3:\penalty0 422--440, 2021.

\bibitem[Krishnapriyan et~al.(2021)Krishnapriyan, Gholami, Zhe, Kirby, and
  Mahoney]{krishnapriyan2021characterizing}
A.~S. Krishnapriyan, A.~Gholami, S.~Zhe, R.~M. Kirby, and M.~W. Mahoney.
\newblock Characterizing possible failure modes in physics-informed neural
  networks.
\newblock In M.~Ranzato, A.~Beygelzimer, Y.~Dauphin, P.~Liang, and J.~W.
  Vaughan, editors, \emph{Advances in Neural Information Processing Systems},
  volume~34, pages 26548--26560. Curran Associates, Inc., 2021.

\bibitem[Kurz et~al.(2022)Kurz, De~Gersem, Galetzka, Klaedtke, Liebsch,
  Loukrezis, Russenschuck, and Schmidt]{kurz2022hybrid}
S.~Kurz, H.~De~Gersem, A.~Galetzka, A.~Klaedtke, M.~Liebsch, D.~Loukrezis,
  S.~Russenschuck, and M.~Schmidt.
\newblock Hybrid modeling: {T}owards the next level of scientific computing in
  engineering.
\newblock \emph{Journal of Mathematics in Industry}, 12, 2022.

\bibitem[Lui(2009)]{LUI2009541}
S.~Lui.
\newblock Spectral domain embedding for elliptic {PDE}s in complex domains.
\newblock \emph{Journal of Computational and Applied Mathematics},
  225:\penalty0 541--557, 2009.

\bibitem[Matthysen and Huybrechs(2016)]{Matthysen2016fast}
R.~Matthysen and D.~Huybrechs.
\newblock Fast algorithms for the computation of {F}ourier extensions of
  arbitrary length.
\newblock \emph{SIAM Journal on Scientific Computing}, 38:\penalty0 A899--A922,
  2016.

\bibitem[McGreivy and Hakim(2024)]{mcgreivy2024weakbaselinesreportingbiases}
N.~McGreivy and A.~Hakim.
\newblock Weak baselines and reporting biases lead to overoptimism in machine
  learning for fluid-related partial differential equations.
\newblock \emph{arXiv:2407.07218}, 2024.

\bibitem[Meuris et~al.(2023)Meuris, Qadeer, and Stinis]{meuris2023machine}
B.~Meuris, S.~Qadeer, and P.~Stinis.
\newblock Machine-learning-based spectral methods for partial differential
  equations.
\newblock \emph{Scientific Reports}, 13:\penalty0 1739, 2023.

\bibitem[Mishra and Molinaro(2023)]{mishra2022generalization}
S.~Mishra and R.~Molinaro.
\newblock Estimates on the generalization error of physics-informed neural
  networks for approximating {PDE}s.
\newblock \emph{IMA Journal of Numerical Analysis}, 43:\penalty0 1--43, 2023.

\bibitem[Nelsen and Stuart(2024)]{nelsen2024operator}
N.~H. Nelsen and A.~M. Stuart.
\newblock Operator learning using random features: A tool for scientific
  computing.
\newblock \emph{SIAM Review}, 66:\penalty0 535--571, 2024.

\bibitem[Rahimi and Recht(2007)]{rahimi2007random}
A.~Rahimi and B.~Recht.
\newblock Random features for large-scale kernel machines.
\newblock In J.~Platt, D.~Koller, Y.~Singer, and S.~Roweis, editors,
  \emph{Advances in Neural Information Processing Systems}, volume~20. Curran
  Associates, Inc., 2007.

\bibitem[Raissi et~al.(2019)Raissi, Perdikaris, and
  Karniadakis]{raissi2019PINN}
M.~Raissi, P.~Perdikaris, and G.~Karniadakis.
\newblock Physics-informed neural networks: {A} deep learning framework for
  solving forward and inverse problems involving nonlinear partial differential
  equations.
\newblock \emph{Journal of Computational Physics}, 378:\penalty0 686--707,
  2019.

\bibitem[Rathore et~al.(2024)Rathore, Lei, Frangella, Lu, and
  Udell]{rathore2024challengestrainingpinnsloss}
P.~Rathore, W.~Lei, Z.~Frangella, L.~Lu, and M.~Udell.
\newblock Challenges in training {PINN}s: A loss landscape perspective.
\newblock \emph{arXiv:2402.01868}, 2024.

\bibitem[Rudin(1991)]{rudin1991functional}
W.~Rudin.
\newblock \emph{Functional Analysis}.
\newblock International series in Pure and Applied Mathematics. McGraw-Hill,
  New-York, 2 edition, 1991.

\bibitem[Schaback and Wendland(2006)]{schaback2006from}
R.~Schaback and H.~Wendland.
\newblock Kernel techniques: From machine learning to meshless methods.
\newblock \emph{Acta Numerica}, 15:\penalty0 543--639, 2006.

\bibitem[Shin et~al.(2020)Shin, Darbon, and Karniadakis]{shin2020convergence}
Y.~Shin, J.~Darbon, and G.~E. Karniadakis.
\newblock On the convergence of physics informed neural networks for linear
  second-order elliptic and parabolic type {PDEs}.
\newblock \emph{Communications in Computational Physics}, 28:\penalty0
  2042--2074, 2020.

\bibitem[Shin et~al.(2023)Shin, Zhang, and Karniadakis]{shin2023error}
Y.~Shin, Z.~Zhang, and G.~E. Karniadakis.
\newblock Error estimates of residual minimization using neural networks for
  linear {PDE}s.
\newblock \emph{Journal of Machine Learning for Modeling and Computing},
  4:\penalty0 73--101, 2023.

\bibitem[Tsybakov(2009)]{tsybakov2009introduction}
A.~Tsybakov.
\newblock \emph{Introduction to Nonparametric Estimation}.
\newblock Springer, New York, 2009.

\bibitem[Wang et~al.(2022{\natexlab{a}})Wang, Yu, and Perdikaris]{wang2022when}
S.~Wang, X.~Yu, and P.~Perdikaris.
\newblock When and why {PINN}s fail to train: {A} neural tangent kernel
  perspective.
\newblock \emph{Journal of Computational Physics}, 449:\penalty0 110768,
  2022{\natexlab{a}}.

\bibitem[Wang et~al.(2022{\natexlab{b}})Wang, Xing, Kirby, and
  Zhe]{wang2022physics}
Z.~Wang, W.~Xing, R.~Kirby, and S.~Zhe.
\newblock Physics informed deep kernel learning.
\newblock \emph{arXiv:2006.04976}, 2022{\natexlab{b}}.

\bibitem[Yang et~al.(2012)Yang, Li, Mahdavi, Jin, and Zhou]{yang2012nystrom}
T.~Yang, Y.-F. Li, M.~Mahdavi, R.~Jin, and Z.-H. Zhou.
\newblock Nystr\"{o}m method vs random {F}ourier features: A theoretical and
  empirical comparison.
\newblock In F.~Pereira, C.~Burges, L.~Bottou, and K.~Weinberger, editors,
  \emph{Advances in Neural Information Processing Systems}, volume~25. Curran
  Associates, Inc., 2012.

\end{thebibliography}

\end{document}